\documentclass[sigconf,nonacm]{acmart}
 
\AtBeginDocument{%
  }
\usepackage{caption}
\usepackage{subcaption}
\usepackage{tikz}
\usetikzlibrary{arrows.meta, positioning, fit, backgrounds, shapes.geometric, calc, matrix, decorations.pathreplacing, shapes.arrows}
\usepackage{adjustbox}
\usetikzlibrary{positioning, shapes, arrows.meta, backgrounds}

\usepackage{amsthm}
\newtheorem{theorem}{Theorem}[section]
\newtheorem{definition}{Definition}[section]
\usepackage{filecontents}
\usepackage{hyperref}
\usepackage{booktabs}
\usepackage{multirow}
\usepackage{tabularx}
\definecolor{GainLight}{HTML}{D6EFEA}   
\definecolor{GainMed}{HTML}{6FBFA8}     
\definecolor{GainStrong}{HTML}{3DA088}  
\definecolor{LossLight}{HTML}{FBE3D0}   
\definecolor{LossMed}{HTML}{F5C293}     
\definecolor{LossStrong}{HTML}{EFA15F}  
\usepackage[table]{xcolor}      
 
\renewcommand\footnotetextcopyrightpermission[1]{} 
\begin{document}
 
\title{Finding Connections: Membership Inference Attacks for the Multi-Table Synthetic Data Setting}
 
\author{Joshua Ward}
\affiliation{%
  \institution{University of California, Los Angeles}
  \city{Los Angeles}
  \state{CA}
  \country{USA}
}
 
\author{Chi-Hua Wang}
\affiliation{%
  \institution{Purdue University}
  \city{West Lafayette}
  \state{IN}
  \country{USA}
}
 
\author{Guang Cheng}
\affiliation{%
  \institution{University of California, Los Angeles}
  \city{Los Angeles}
  \state{CA}
  \country{USA}
}

\renewcommand{\shortauthors}{Ward et al.}

\begin{abstract}
Synthetic tabular data has gained attention for enabling privacy-preserving data sharing. While substantial progress has been made in single-table synthetic generation where data are modeled at the row or item level, most real-world data exists in relational databases where a user's information spans items across multiple interconnected tables. Recent advances in synthetic relational data generation have emerged to address this complexity, yet release of these data introduce unique privacy challenges as information can be leaked not only from individual items but also through the relationships that comprise a complete user entity. 

To address this, we propose a novel Membership Inference Attack (MIA) setting to audit the empirical user-level privacy of synthetic relational data and show that single-table MIAs that audit at an item level underestimate user-level privacy leakage. We then propose Multi-Table Membership Inference Attack (MT-MIA), a novel adversarial attack under a No-Box threat model that targets learned representations of user entities via Heterogeneous Graph Neural Networks. By incorporating all connected items for a user, MT-MIA better targets user-level vulnerabilities induced by inter-tabular relationships than existing attacks. We evaluate MT-MIA on a range of real-world multi-table datasets and demonstrate that this vulnerability exists in state-of-the-art relational synthetic data generators, employing MT-MIA to additionally study where this leakage occurs. 
\end{abstract}



\maketitle
\section{Introduction}

Synthetic tabular data has been shown to enable the sharing of sensitive or private information \cite{yoon2018pategan,yoon2020anonymization}. While considerable progress in synthetic data generation has focused on single table applications, where a generative model learns the distribution of a single table, most real world data exists in hierarchical structures stored in relational databases, where rows in one table have interdependencies with rows in other tables \cite{Fayyad_Piatetsky-Shapiro_Smyth_1996, MartnezCruz2012OntologiesVR, pmlr-v235-fey24a}. Modeling and producing synthetic databases rather than single tables has seen growing interest as it allows for the release of more expressive data and lately a variety of methods have been developed to learn and generate synthetic databases \cite{sdv, Padhi, Gueye, solatorio2023realtabformer, clavaddpm, hudovernik2025reldiffrelationaldatagenerative}.

While the results of these models are impressive, auditing the empirical privacy of synthetic database release is not well understood. Unlike single table settings that protect privacy at a row or item level, in relational databases a user's information is distributed across items in multiple interconnected tables. As we will show in Sections~\ref{subsec:motivexample} and~\ref{sec:exp}, privacy leakage that occurs over a particular item implies the leakage of all connected items.

Membership Inference Attacks (MIAs) have been successfully applied to audit the privacy of single table tabular generators and can be used to estimate the empirical $\epsilon$ differential privacy of generated synthetic data \cite{vanbreugel2023membership, groundhog}. However, we demonstrate that single table MIAs are inadequate for auditing multi table synthetic data privacy at the user level because they can only audit privacy at the item level. Current attacks, while applicable to item level privacy auditing in individual tables within a synthetic database, fail to exploit the critical inter tabular relationships that define a user across multiple tables, rendering them unsuitable for comprehensive user level privacy auditing in relational database contexts.

To our knowledge the first work to study this problem, we propose a novel Membership Inference Attack setting designed to audit the \textit{user level privacy} of multi table synthetic data release. Here, we construct relational databases as heterogeneous graphs in which we infer if a test \textit{subgraph} describing all connected information for a user was included in the training set of the generative model that produced the synthetic database. We then propose a new attack, Multi Table MIA (MT-MIA), which performs membership inference by learning graphical representations of user centric subgraphs using heterogeneous graph neural networks (HGNNs) \cite{10.1145/3292500.3330961, 10.1145/3308558.3313562, 10.1145/3366423.3380297, fey2023relationaldeeplearninggraph, yang2023simpleefficientheterogeneousgraph, robinson2024relbench} under a No-Box threat model.

Unlike existing tabular MIAs, MT-MIA leverages all relational information associated with a user, explicitly targeting inter-table conditional dependencies that are inaccessible to single table formulations. The attack is model agnostic and can be applied to arbitrary multi table datasets and synthetic data generators, making it broadly applicable to both practitioners and researchers.

To validate MT-MIA, we construct examples of multi table privacy leakage and show that current single table, item level attacks are no better than random guessing at user level membership inference in these scenarios whereas MT-MIA achieves near perfect AUC, highlighting the need for user level specific privacy auditing techniques. We then empirically deploy MT-MIA on a variety of real world multi table datasets, finding that current state of the art multi table generators possess this unique vulnerability, even under a conservative threat model. Finally, we analyze the intermediate embedding spaces of MT-MIA to show that MT-MIA can diagnose where memorization may be occurring in the multi table training set. Overall, these results demonstrate that effective privacy auditing for multi table generative models requires user level analyses, and that MT-MIA provides a practical mechanism for uncovering such leakage in real world settings.
\section{Related Works}
\subsection{Synthetic Data Generation and Release}

\textbf{Non-Relational Tabular Data}: The objective of synthetic data generation is to learn the probability distribution of a training dataset from which to generate new artificial samples that exhibit statistical properties similar to the original data. In the {non-relational tabular data} case, each row typically represents a complete entity (e.g., a single user with all their associated attributes), with each entity modeled as an independent observation or item associated with a fixed set of features. Techniques such as generative adversarial networks \cite{Xu2019ModelingTD,yoon2018pategan,yoon2020anonymization}, language models \cite{borisov2023languagemodelsrealistictabular, solatorio2023realtabformer}, and diffusion models \cite{tabddpm,autodiff,tabsyn}, have demonstrated an impressive ability to generate realistic and diverse synthetic data.

\textbf{Relational Tabular Data:} While single-table approaches are useful in certain applications, most data of release interest in healthcare, education, finance, and government are not stored as isolated tables but rather in relational databases \cite{MartnezCruz2012OntologiesVR,fey2023relationaldeeplearninggraph,pmlr-v235-fey24a}. 

\begin{definition}[Relational database]
\label{def:relational_db}
A relational database $\mathcal{D} = (\mathcal{T}, \mathcal{J})$ consists of a collection of tables $\mathcal{T} = \{T_1, \dots, T_n\}$ and joins or links between these tables $\mathcal{J} \subseteq \mathcal{T} \times \mathcal{T}$. Each table is a set $T' = \{o_1, \dots, o_n\}$ where the elements $o \in \mathcal{O}$ are referred to as rows or observations. Each observation is a tuple $o = (\mathcal{P}_o, \mathcal{K}_o, f_o)$ where:
\begin{itemize}
    \item $\mathcal{P}_o$ is the Primary Key that uniquely identifies the observation $o$.
    \item $\mathcal{K}_o$ is the set of Foreign Keys corresponding to a primary key in other tables, thus connecting the tables.
    \item $f_o$ corresponds to the features or columns of the observation $o$.
\end{itemize}
\end{definition}

In contrast to the single-table case, relational tabular data distributes information about a single user or entity across items in multiple tables connected through conditional joint relationships. As relational tabular data can be more expressive than its non-relational counterpart, a number of methods have been proposed to learn and generate synthetic relational data using probabilistic methods \cite{sdv, Gueye}, transformers \cite{Padhi}, latent diffusion \cite{clavaddpm}, and language models \cite{solatorio2023realtabformer}. A common theme across these approaches is to define chains of modular generators for each table based on the parent-child relationships implied by the join relationships in a database schema. The typical generation process begins by synthesizing data for parent tables, then recursively and conditionally generating their children while controlling cardinality and join relationships through histogram-based or clustering-based mechanisms.

The release of these data raises unique privacy challenges compared to the non-relational setting as the \textit{granularity} of the unit of privacy the releasing party wishes to protect can change. In single-table scenarios, one independent entity or user is a row or item. In contrast, relational data structures represent a user as a set of inter-related items. This interconnectedness means that protecting privacy is no longer confined to an \textit{item-level} but rather a \textit{user-level}. As we will show in Section  \ref{sec:setting}, leakage of an item or join relationship implies privacy leakage about all related rows for a user, and current item-level auditing procedures underestimate this risk.

\subsection{Membership Inference Attacks for Synthetic Data Generation}

 MIAs are a class of adversarial techniques that aim to distinguish between member records—those used in the training set of a target model—and nonmember records—those drawn from an independent dataset \cite{Shokri}. Originally proposed in the context of classification models \cite{Sablayrolles2019WhiteboxVB,Long,Carlini2021MembershipIA,watson2022on, ye2022,zarifzadeh2024lowcosthighpowermembershipinference}, MIAs exploit the tendency of learning algorithms to behave differently on training data than on unseen data, often due to overfitting or memorization. As a result, MIAs have become a central tool for empirically auditing privacy leakage in machine learning systems, complementing formal guarantees such as differential privacy. More recently, MIAs have been adapted to the setting of tabular synthetic data generation, where the adversary’s goal is to infer whether a record contributed to the training of a generative model based on access to synthetic samples. In this context, successful membership inference indicates that the synthetic data preserves information too faithfully, potentially enabling privacy violations even when direct record linkage is not possible.
\subsubsection{Single-Table MIAs}

Membership inference attacks against single-table synthetic data aim to determine whether a given real record influenced the training of a generative model, based solely on properties of the released synthetic dataset and, in some cases, limited auxiliary information. Unlike MIAs for discriminative models, where the attacker probes a target model's outputs directly, attacks in the synthetic data setting must infer membership indirectly through distributional artifacts left behind by the generation process.

A broad class of attacks operates using only the released synthetic table and, in some cases, auxiliary reference data, without access to the generative model. These attacks exploit the observation that records drawn from the training set often induce locally atypical behavior in the synthetic distribution, such as elevated neighborhood density, reduced variability, or near-duplicate synthetic samples. Operationally, they measure the extent a candidate record is memorized or overfit too using using distances, density estimates, or reconstruction scores to classify record membership \cite{Hayes2017LOGANMI,Hilprecht2019MonteCA,ganleaks,ward2024dataplagiarismindexcharacterizing,vanbreugel2023membership,Gen-LRA}. By relying only on released samples and optional auxiliary data, these methods avoid assumptions about access to the generative model. Consequently, they are typically classified as operating in a no-box threat model. This setting is particularly relevant for private data release, where a data curator is unlikely to disclose generative model information to a potential adversary.

Stronger attacks assume access to additional information, most commonly partial or query access to the generative model. These attacks explicitly compare how likely a target record is under competing member and non-member hypotheses, often by training multiple shadow generative models to approximate each hypothesis and learning a decision rule to distinguish between them \cite{groundhog,houssiau2022tapas,Meeus_2024}. While such attacks can substantially improve inference accuracy, they are typically extremely computationally expensive, requiring repeated model training or large numbers of model queries. As a result, their practical relevance to synthetic data release is less clear, since data curators generally do not expose generative models or provide the level of access required to support such attacks.

Across this literature, attacks are formulated to audit item-level privacy for non-relational, single-table generative models, implicitly assuming that each record corresponds to an independent, fixed-dimensional feature vector. This assumption enables distance and density-based attacks and ties attack effectiveness directly to the degree of over-representation or memorization of individual training records in the synthetic data. In contrast, multi-table synthetic data induces a hierarchical or relational representation in which membership corresponds to the presence of a training entity across multiple linked tables. Attacks in this setting must implicitly learn a representation over sets of rows or relational subgraphs, rather than operating on a single row input.

\subsubsection{Multi-Table MIAs}
Compared to the single-table setting, membership inference attacks for multi-table synthetic data remain relatively underexplored. To date, there has been limited effort to formalize a threat model or define a standard MIA setting of entity membership across multiple relational tables.

The primary empirical study in this setting is the MIDST Competition at SATML 2025 \cite{shafieinejad2026midstchallengesatml2025}, which evaluated membership inference attacks against multi-table synthetic data generated by ClavaDDPM under a range of white-box and black-box threat models. Despite access to multiple linked synthetic tables, the competition hosts noted that the strongest attacks relied exclusively on a designated main table, effectively reducing the problem to a single-table setting. Indeed, the winner of the competition ignored auxiliary tables entirely \cite{wu2025winningmidstchallengenew}. Attacks that attempted to incorporate information from auxiliary tables or relational structure were noted to not yield improved performance and that they often performed worse than single-table strategies.

Overall, existing empirical results suggest that current membership inference techniques do not successfully exploit relational dependencies in multi-table synthetic data, and that effective attacks in this setting have yet to be demonstrated.

\subsection{Heterogeneous Graph Neural Networks}
Heterogeneous Graph Neural Networks (HGNNs) provide the necessary inductive bias to preserve the multi-modal semantics of relational databases by explicitly modeling the distinct node and edge types inherent in database schemas. Unlike homogeneous GNNs, which treat all connections as semantically equivalent, HGNNs utilize type-specific transformation matrices and message-passing protocols to navigate the "web" of relational tables \cite{schlichtkrull2018modeling}. This capability is essential for membership inference in multi-table contexts, as privacy leakage often resides not in a single row, but in the specific structural alignment between a parent entity and its conditionally generated child records.

Early architectures in this domain, such as the Heterogeneous Graph Attention Network \cite{10.1145/3308558.3313562}, relied on hierarchical attention mechanisms—operating at both the node and semantic levels—to aggregate information across predefined meta-paths. However, the requirement for manual path engineering often limits their flexibility in complex database schemas. To address this, more recent frameworks like the Heterogeneous Graph Transformer \cite{10.1145/3366423.3380027} and the Graph Attention Transformer operator \cite{gatv2} introduce dynamic, typed-attention mechanisms that automatically learn the importance of different relational dependencies.

The recent formalization of Relational Deep Learning \cite{pmlr-v235-fey24a} further validates the use of HGNNs for database-centric tasks, demonstrating that structural representations can effectively capture the joint distributions of tabular data without the loss of information inherent in table flattening. MT-MIA leverages these advancements to perform privacy auditing, utilizing HGNNs to detect "memorized motifs"—instances where a synthetic generator reproduces training structures.

\section{A User-Level Membership Inference Attack Setting}
\label{sec:setting}

Releasing relational synthetic data implies a notion of privacy at the user-level which differs from the item-level privacy implied in the single-table setting. In this section, we develop a novel MIA setting to audit the empirical privacy of user entities and show that item-level privacy does not necessarily provide privacy protection at the user level, necessitating MT-MIA, a new membership inference technique proposed in Section \ref{sec:methodology} that follows this setting.
\subsection{Relational Databases as Heterogeneous Graphs}
\label{sec:rel_to_hg}

We represent a relational database $\mathcal{D}$ as a heterogeneous graph in order to reason about user level entities as connected subgraphs rather than isolated rows. This representation makes cross table dependencies explicit and is what allows our attack to operate on the full set of records belonging to a user.

\begin{definition}[Database as a heterogeneous graph]
\label{def:db_as_hg}
Given a database $\mathcal{D} = (\mathcal{T}, \mathcal{J})$, we construct a heterogeneous graph $G = (\mathcal{V}, \mathcal{E}, \mathcal{T}_V, \mathcal{T}_E, X_V)$ as follows:
\begin{itemize}
    \item Each table $T_i \in \mathcal{T}$ defines a node type $t_i \in \mathcal{T}_V$. Each row $o \in T_i$ becomes a node $v \in \mathcal{V}$, with type assigned by $\phi_V: \mathcal{V} \to \mathcal{T}_V$, $\phi_V(v) = t_i$.
    \item The non key features $f_o$ of a row form the node's attribute vector $x_v \in X_V$.
    \item Each Foreign Key relation in $\mathcal{J}$ defines an edge type $r \in \mathcal{T}_E$. A Foreign Key reference from row $u$ to row $v$ instantiates a directed, typed edge $e = (u, r, v) \in \mathcal{E}$.
\end{itemize}
\end{definition}

\noindent The schema is preserved as a typed meta structure: the set of admissible $(\phi_V(u), r, \phi_V(v))$ triples is exactly the set of Foreign Key relations in $\mathcal{D}$. A user entity, a set of rows transitively linked via Foreign Key relations, corresponds to a connected, typed subgraph of $G$. This graph perspective is particularly valuable for our membership inference context, as it enables us to trace information leakage across table boundaries and identify how patterns in relational data might reveal user membership despite protections that may be effective at the individual table level. For the remainder of the paper, we use graph language interchangeably with database language: nodes refer to rows, node types to tables, edges to Foreign Key references, and edge types to Foreign Key relations.

\begin{figure*}[t]
    \centering
    \begin{subfigure}[b]{0.24\textwidth}
        \centering
        \begin{tikzpicture}[
    customer/.style={circle, draw=black, thick, minimum size=0.8cm, fill=white},
    customer_used/.style={circle, draw=blue!70!black, thick, minimum size=0.8cm, fill=blue!20},
    transaction/.style={circle, draw=black, minimum size=0.6cm, fill=white},
    transaction_used/.style={circle, draw=blue!70!black, minimum size=0.6cm, fill=blue!20},
    transaction_faded/.style={circle, draw=gray!40, minimum size=0.6cm, fill=gray!10},
    edge/.style={->, >=Stealth, thick},
    edge_used/.style={->, >=Stealth, thick, blue!70!black},
    edge_faded/.style={->, >=Stealth, gray!40},
    title/.style={font=\bfseries\large},
    notation/.style={font=\small\ttfamily, align=left},
    label/.style={font=\small, align=center}
]
    \node[customer_used] (C1a) at (0, 0) {$C_1$};
    \node[transaction_faded] (T1a) at (2, 1) {$T_1$};
    \node[transaction_faded] (T2a) at (2, 0) {$T_2$};
    \node[transaction_faded] (T3a) at (2, -1) {$T_3$};
    
    \draw[edge_faded] (C1a) -- (T1a);
    \draw[edge_faded] (C1a) -- (T2a);
    \draw[edge_faded] (C1a) -- (T3a);
    
    \node[notation, anchor=north] at (1, -2) {$h^* = \{x_c \mid c \in \mathcal{C}\}$};
\end{tikzpicture}
        \caption{Customer Features Only}
        \label{fig:motive_a}
    \end{subfigure}
    \hfill
    \begin{subfigure}[b]{0.24\textwidth}
        \centering
        \begin{tikzpicture}[
    customer/.style={circle, draw=black, thick, minimum size=0.8cm, fill=white},
    customer_used/.style={circle, draw=blue!70!black, thick, minimum size=0.8cm, fill=blue!20},
    transaction/.style={circle, draw=black, minimum size=0.6cm, fill=white},
    transaction_used/.style={circle, draw=blue!70!black, minimum size=0.6cm, fill=blue!20},
    transaction_faded/.style={circle, draw=gray!40, minimum size=0.6cm, fill=gray!10},
    edge/.style={->, >=Stealth, thick},
    edge_used/.style={->, >=Stealth, thick, blue!70!black},
    edge_faded/.style={->, >=Stealth, gray!40},
    title/.style={font=\bfseries\large},
    notation/.style={font=\small\ttfamily, align=left},
    label/.style={font=\small, align=center}
]
    \node[customer_used] (C1c) at (0, 0) {$C_1$};
    \node[transaction_used] (T1c) at (2, 1) {$T_1$};
    \node[transaction_faded] (T2c) at (2, 0) {$T_2$};
    \node[transaction_faded] (T3c) at (2, -1) {$T_3$};
    
    \draw[edge_used] (C1c) -- (T1c);
    \draw[edge_faded] (C1c) -- (T2c);
    \draw[edge_faded] (C1c) -- (T3c);
    
    \node[notation, anchor=north] at (1, -2) {$h^* = \{(x_c, x_t) \mid t \in (\mathcal{T}_c)\}$};
\end{tikzpicture}
        \caption{Join to Single Transaction}
        \label{fig:motive_b}
    \end{subfigure}
    \hfill
    \begin{subfigure}[b]{0.24\textwidth}
        \centering
        \begin{tikzpicture}[
    customer/.style={circle, draw=black, thick, minimum size=0.8cm, fill=white},
    customer_used/.style={circle, draw=blue!70!black, thick, minimum size=0.8cm, fill=blue!20},
    transaction/.style={circle, draw=black, minimum size=0.6cm, fill=white},
    transaction_used/.style={circle, draw=blue!70!black, minimum size=0.6cm, fill=blue!20},
    transaction_faded/.style={circle, draw=gray!40, minimum size=0.6cm, fill=gray!10},
    edge/.style={->, >=Stealth, thick},
    edge_used/.style={->, >=Stealth, thick, blue!70!black},
    edge_faded/.style={->, >=Stealth, gray!40},
    title/.style={font=\bfseries\large},
    notation/.style={font=\small\ttfamily, align=left},
    label/.style={font=\small, align=center}
]
    \node[customer_used] (C1b) at (0, 0) {$C_1$};
    \node[transaction_faded] (T1b) at (2, 1) {$T_1$};
    \node[transaction_faded] (T2b) at (2, 0) {$T_2$};
    \node[transaction_faded] (T3b) at (2, -1) {$T_3$};
    
    \draw[edge_faded] (C1b) -- (T1b);
    \draw[edge_faded] (C1b) -- (T2b);
    \draw[edge_faded] (C1b) -- (T3b);
    
    \draw[densely dashed, blue!70!black, rounded corners=3pt] 
        (1.5, 1.5) rectangle (2.5, -1.5);
    \node[font=\scriptsize, blue!70!black] at (2, -1.8) {agg};
    
    \node[notation, anchor=north] at (1, -2) {$h^* = \{x_c \oplus \text{agg}(\{x_t\})\}$};
\end{tikzpicture}
        \caption{Transaction Aggregation}
        \label{fig:motive_c}
    \end{subfigure}
    \hfill
    \begin{subfigure}[b]{0.24\textwidth}
        \centering
        \begin{tikzpicture}[
    customer/.style={circle, draw=black, thick, minimum size=0.8cm, fill=white},
    customer_used/.style={circle, draw=blue!70!black, thick, minimum size=0.8cm, fill=blue!20},
    transaction/.style={circle, draw=black, minimum size=0.6cm, fill=white},
    transaction_used/.style={circle, draw=blue!70!black, minimum size=0.6cm, fill=blue!20},
    transaction_faded/.style={circle, draw=gray!40, minimum size=0.6cm, fill=gray!10},
    edge/.style={->, >=Stealth, thick},
    edge_used/.style={->, >=Stealth, thick, blue!70!black},
    edge_faded/.style={->, >=Stealth, gray!40},
    title/.style={font=\bfseries\large},
    notation/.style={font=\small\ttfamily, align=left},
    label/.style={font=\small, align=center}
]
    \node[customer_used] (C1d) at (0, 0) {$C_1$};
    \node[transaction_used] (T1d) at (2, 1) {$T_1$};
    \node[transaction_used] (T2d) at (2, 0) {$T_2$};
    \node[transaction_used] (T3d) at (2, -1) {$T_3$};
    
    \draw[edge_used] (C1d) -- (T1d);
    \draw[edge_used] (C1d) -- (T2d);
    \draw[edge_used] (C1d) -- (T3d);
    
    \node[notation, anchor=north] at (1, -2) {$h^* = (\mathcal{V}_C \cup \mathcal{V}_T, \mathcal{E})$};
\end{tikzpicture}
        \caption{MT-MIA (Ours)}
        \label{fig:motive_d}
    \end{subfigure}
        \caption{{Comparison of attack strategies for Section~\ref{subsec:motivexample}.} All methods operate on the same Customer-Transaction graph ($C_1$ with three transactions $T_1, T_2, T_3$). Blue highlighting indicates information used by each attack; gray indicates ignored or collapsed information. Single-table based attacks have three options in constructing their attack: \textbf{(a)} Use only customer features, ignoring all relational information. \textbf{(b)} Join each customer to a single transaction, requiring arbitrary sampling and discarding remaining relationships. \textbf{(c)} Joining customers to  arbitrarily aggregated transaction features, collapsing the relationship structure. \textbf{(d)} MT-MIA preserves and learns from the complete graph structure, automatically discovering that relationship cardinality reveals membership (AUC = 0.999).}
    \label{fig:motive_ex}
    \Description{Four schematic diagrams labeled (a) through (d), each showing a Customer-Transaction graph with one customer node C1 connected to three transaction nodes T1, T2, and T3. (a) Customer Features Only: only the customer node C1 is highlighted in blue; the three transaction nodes and their edges are gray, indicating that this attack uses only customer features and ignores all transaction data. (b) Join to Single Transaction: customer node C1 and one of the three transaction nodes are highlighted in blue, with the other two transactions and their edges in gray, showing that this attack arbitrarily selects one transaction per customer and discards the rest. (c) Transaction Aggregation: customer node C1 is highlighted in blue and connects to a single aggregation node, which itself connects to all three transaction nodes; the original transaction-to-customer edges are gray, indicating that the three transactions are collapsed into one aggregated representation. (d) MT-MIA: all four nodes (C1, T1, T2, T3) and all edges between them are highlighted in blue, indicating that the attack uses the complete graph structure including the cardinality of transactions per customer.}
    \end{figure*}
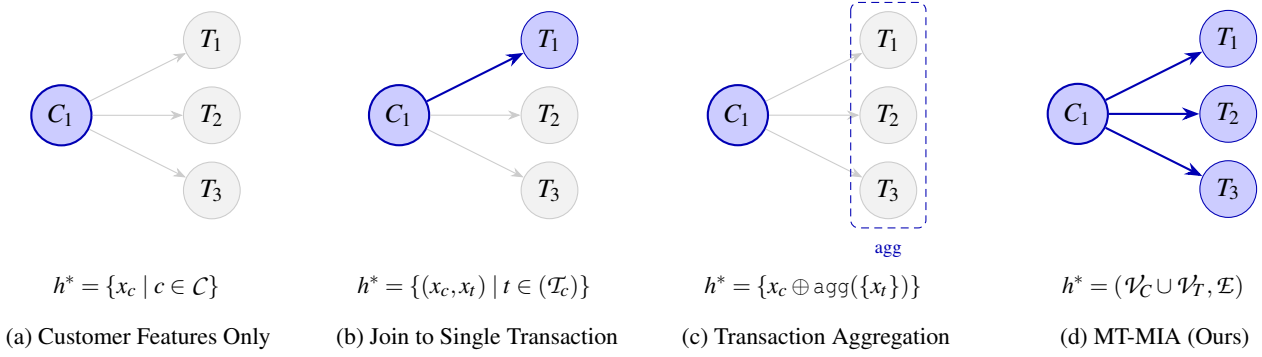

\subsection{Formalism}
\label{sec:formalism}

In the multi table synthetic data generation setting, a generative model $\mathcal{M}$ is trained on a heterogeneous graph $G_{\text{train}} \sim \mathbb{P}$ that is a random sample of the population. $\mathcal{M}$ is then sampled to produce a synthetic graph $G_{\text{synth}}$. We define an entity subgraph $h$ as a connected subgraph of $G$ corresponding to all rows associated with a single user. Following the classical Membership Inference game \cite{Shokri}, we define a test entity subgraph $h^*$ as either a subgraph of $G_{\text{train}}$ or a fresh sample from a holdout graph $G_{\text{holdout}} \sim \mathbb{P}$. Let $\mathcal{H}$ denote the set of all such test entity subgraphs. An adversary $\mathcal{A}: \mathcal{H} \to \{0, 1\}$, with access to $G_{\text{synth}}$ and perhaps other information defined by a threat model, aims to determine the membership for a given $h^*$. Formally, this Membership Inference Attack can be expressed as:\begin{equation}\label{eq:membership_prediction}
    \mathcal{A}(h^*) \;=\; \mathbb{I}\!\left\{ f(h^*) > \gamma \right\},
\end{equation}
where $\mathbb{I}$ is the indicator function, $f$ is a scoring function evaluated on $h^*$ and $G_{\text{synth}}$, and $\gamma$ is an adjustable decision threshold. The success of the attack can be measured using traditional binary classification metrics and can be interpreted as a measure of the privacy leakage introduced by the release of $G_{\text{synth}}$ sampled from the model trained on the original data.

This formulation tests whether $h^* \subseteq G_{\text{train}}$. It differs from traditional MIAs that test if an independent item or row was included in the training dataset; here, we evaluate whether the generative model preserves the privacy of all related nodes for a user $v \in h^*$ when released together.

 \textbf{Example}: Consider a relational database with two tables: \texttt{Customers} and \texttt{Transactions}, where each customer may place multiple distinct transactions (one-to-many relationship). This schema induces a heterogeneous graph in which each customer node connects through edges to a set of transaction nodes, forming a connected subgraph for each customer. In this setting, our task is to infer whether such a subgraph, representing a distinct customer and their full transaction history, was used for training.

\subsection{Privacy Auditing with Subgraphs}
\label{sec:auditing_subgraphs}

In theory, $h^*$ could be any subgraph of interest to an adversary or auditor, since if $h^* \subseteq G_{\text{train}}$, then $\forall g \subseteq h^*, g \subseteq G_{\text{train}}$. However, it is useful to add several conditions to simplify auditing procedures, as conducting membership inference on all possible constructions of $h^*$ is often computationally infeasible. First, we restrict $h^*$ to be a connected subgraph, as by the construction of relational databases two unconnected nodes imply independence. Second, many real world databases naturally decompose into disjoint connected subgraphs, where each subgraph represents a unit or entity, such as a customer and their associated transactions. We thus propose auditing a finite, well structured set of subgraphs by leveraging the relational decomposition of databases.

Let $H_{\text{test}} = \{ h_1, h_2, \dots, h_n \}$ denote the set of all disjoint connected entity subgraphs in $G_{\text{test}} = G_{\text{train}} \cup G_{\text{holdout}}$, where $h_i \subseteq G_{\text{test}}$ and $h_i \cap h_j = \emptyset$ for $i \neq j$. We propose to audit membership specifically over $H_{\text{test}}$. This formulation and these assumptions lead to the following result:

\begin{theorem}
\label{thm:leakage}
Let $G_{\text{test}} = (V, E)$ be a graph that is the disjoint union of two subgraphs $G_{\text{train}}$ and $G_{\text{holdout}}$, where $V(G_{\text{test}}) = V(G_{\text{train}}) \cup V(G_{\text{holdout}})$ and $V(G_{\text{train}}) \cap V(G_{\text{holdout}}) = \emptyset$. Furthermore, there are no edges in $G$ connecting vertices between $G_{\text{train}}$ and $G_{\text{holdout}}$ (i.e., all edges in $G$ have both endpoints in either $G_{\text{train}}$ or in $G_{\text{holdout}}$). Let $h^* \subseteq G$ be a connected subgraph, and let $g \subseteq h^*$.

Then, if $g \subseteq G_{\text{train}}$, it follows that $h^* \subseteq G_{\text{train}}$. Likewise, if $g \subseteq G_{\text{holdout}}$, then $h^* \subseteq G_{\text{holdout}}$.
\end{theorem}
A proof is included in Appendix~\ref{app:proofs}. As a sketch, by construction $h^*$ cannot have nodes nor edges that connect to nodes in both $G_{\text{train}}$ and $G_{\text{holdout}}$. This exclusivity establishes that the membership of all nodes in $g$ and $h^*$ cannot differ. An immediate corollary is that under the disjointness assumption, $h^* \subseteq G_{\text{train}}$ and $h^* \subseteq G_{\text{holdout}}$ are mutually exclusive: every candidate subgraph in $H_{\text{test}}$ has a well defined membership label.

We note that this auditing setup aligns with how multi table synthetic data generators are trained in practice. Generators must be fit on entity subgraphs that preserve the full set of rows belonging to each user, because the joint distribution of a user across tables is precisely what they are designed to model; partial subgraphs would distort the conditional dependencies between parent and child rows and yield a generator that is unfaithful to the source distribution. Membership at the level of complete entity subgraphs is therefore the natural unit of inference for relational synthetic data release.

\textbf{Consequence of Multi Table Synthetic Data Release.} While the assumption of Theorem~\ref{thm:leakage} is not strictly required for relational data MIAs, it emphasizes the privacy risk of relational synthetic data release: if privacy leakage occurs over any component of a connected subgraph, it implies that the entire subgraph (and all included nodes) must be a member of the same source graph. In other words, it is not enough to protect the privacy of the observations of one individual table, as \textit{all connected information can risk membership inference}.

\subsection{Threat Model}
\label{sec:threat_model}

In this work, we explore membership inference attacks under a No-Box threat model. In the No-Box setting, the attacker has access only to a single synthetic dataset $G_{\text{synth}}$ as well as a database schema and must reason about membership without any knowledge of the generator architecture, training procedure, or internal parameters. This threat model reflects scenarios where a party has published a synthetic dataset in isolation with no additional model implementation details, and an adversary must assess privacy risks based solely on patterns and statistical properties present in $G_{\text{synth}}$.

While a variety of other threat models have been studied for synthetic tabular data, including Calibrated No-Box where the adversary has an additional reference dataset \cite{vanbreugel2023membership, Gen-LRA} and Shadow-Box where an adversary additionally has implementation knowledge of the generator in order to generate shadow models \cite{groundhog,houssiau2022tapas,Meeus_2024}, No-Box is the threat model that most closely matches how synthetic relational data is released in practice. In typical deployments, a data curator publishes a synthetic dataset without releasing the generator, its training data, or any auxiliary reference data, and an auditor or adversary must reason about membership from the synthetic dataset alone. The other threat models impose assumptions that are implausible in the multi table setting:
\begin{itemize}
    \item \textbf{Calibrated No-Box} assumes the adversary additionally holds a fresh sample from the same population distribution as the training data, which in the relational setting requires an entire reference database that faithfully reproduces the joint distribution over node attributes and edge relationships across all tables. This is a substantially stronger assumption than its single table analog and is implausible in the settings synthetic relational data release is meant to enable, such as healthcare and finance, where the whole reason to release synthetic data is that comparable real data is not available.
    \item \textbf{Shadow-Box} grants the adversary knowledge of the model implementation along with a reference dataset, enabling the construction of shadow models. This is even more implausible in relational synthetic data release: such attacks are trivially defeated by not publishing implementation details, and are computationally infeasible for modern relational generators. In our experimentation, a single training run of RelDiff took 48 hours on an H200 to converge under default hyperparameters.
\end{itemize}
\noindent While No-Box attacks represent a lower bound on the privacy leakage detectable under more powerful threat models, this setting is most operationally relevant as it is the most realistic adversarial setting for released synthetic relational dataset in practice.

Additionally, a No-Box threat model also allows MT-MIA to be both model agnostic and dataset agnostic, enabling straightforward auditing of newly proposed relational data generators as they are developed. Because the attack operates solely on the released synthetic data and schema, it does not require adaptation to generator specific interfaces or assumptions, making it applicable in post hoc privacy evaluations where only the synthetic dataset is available.

\subsection{Motivating Example: The Need for Multi-Table Attacks}\label{subsec:motivexample}

To illustrate how inter-table dependencies can leak privacy, we construct a toy example (full experimental details are included in the Appendix). Consider a database with two tables: \textit{Customers} and \textit{Transactions}, connected by a one-to-many relationship where multiple transactions belong to a single customer. Both tables have identical feature distributions following $\mathcal{N}(0, I)$. 

In our constructed scenario, non-member samples contain exactly one transaction per customer, while member samples (training data) contain 100 transactions per customer—a pattern that might arise from data drift or sampling bias. A synthetic data generator produces $G_{\text{synth}}$ with the same distributional properties as the training set. Our goal is to infer membership of customer entities based solely on the release of $G_{\text{synth}}$ under a No-box threat model.

From an adversarial perspective, this leakage is trivial to exploit: a decision rule that predicts membership based on whether a customer has 100 transactions achieves perfect accuracy. However, existing single-table MIAs applied to this scenario such as Distance to Closest Record \cite{ganleaks} and MC \cite{Hilprecht2019MonteCA} fail to detect this leakage. As shown in Figure~\ref{fig:motive_ex}, single-table approaches must make arbitrary choices about how to incorporate multi-table information. A practitioner using existing single-table MIAs would typically either (Figure~\ref{fig:motive_a}) ignore transaction data entirely , (Figure~\ref{fig:motive_b}) join each customer to a single arbitrarily-chosen transaction , or (Figure~\ref{fig:motive_c}) aggregate transaction features (e.g., computing means or sums). While aggregation approaches \textit{could} capture this signal (such as counting the transactions), this requires apriori adversarial knowledge of the leakage as to which relational aspects matter—an assumption that does not scale to complex schemas or subtle leakage patterns.

In contrast (See Figure~\ref{fig:motiv_ex_tpr}), MT-MIA automatically learns from the full relational structure (Figure~\ref{fig:motive_d}) and achieves an AUC of 0.999 without manual feature engineering. While this extreme example is unlikely in practice, it exemplifies a fundamental principle: \textbf{inter-table relationships can leak membership signal}, and user-level adversarial auditing must account for the full multi-table structure.

\section{Methodology: MT-MIA}
\label{sec:methodology}

The motivating example in Section~\ref{subsec:motivexample} demonstrates a fundamental vulnerability in tabular synthetic data auditing: \textit{single table MIAs are topologically limited}. Even when a generator overfits to or memorizes inter-table correlations or cardinality, single table attacks fail to capture this signal because they lack a mechanism to process non Euclidean relational dependencies.

To address this, we propose {Multi Table Membership Inference Attack} (MT-MIA). MT-MIA is designed to be \textit{schema agnostic}, utilizing a heterogeneous graph encoder to map complex relational structures into a low dimensional embedding space. Instead of relying on manual feature engineering or arbitrary aggregation, MT-MIA leverages graph representation learning to identify discriminative structural patterns directly from $G_{\text{synth}}$.

The intuition for MT-MIA is that by incorporating all of an entity subgraph's information into the learned embedding space, the resulting embeddings will be more discriminative than any individual table's representation alone. We first describe the HGNN backbone that maps relational structures into a latent space, then explain how we train the model and score the membership of target subgraphs. Notation introduced throughout this section is summarized in Appendix~\ref{app:notation}.

\subsection{Graph Encoder Architecture}
\label{sec:encoder}

The core of MT-MIA is a heterogeneous graph encoder $\mathcal{M}_\theta$, parameterized by learnable weights $\theta$ (see Figure~\ref{fig:mt-mia-arch}). The HGNN architecture provides the \textit{topological inductive bias} required to model relational dependencies. Unlike traditional autoencoders, $\mathcal{M}_\theta$ is invariant to specific join relationships as well as table and row cardinality, enabling a unified attack interface across any heterogeneous relational schema.

\paragraph{Heterogeneous Message Passing}
To capture the semantics of the relational schema, we stack $L$ heterogeneous message passing layers. For a node type $t \in \mathcal{T}_V$ at layer $l$, the layer's node feature update is:
\begin{equation}
H^{(t)}_{l+1} = \Phi_l^{(t)}\!\left( H^{(t)}_l, \{ A^{(r)} \}_{r \in \mathcal{T}_E} \right),
\end{equation}
where $H^{(t)}_l$ is the matrix of layer $l$ node embeddings for nodes of type $t$, $\Phi_l^{(t)}$ is a type specific message passing function, and $A^{(r)}$ is the adjacency matrix for relation type $r$. Throughout the paper, we instantiate $\Phi_l^{(t)}$ as GATv2 \cite{gatv2} due to its expressive attention mechanism. This mechanism enables information propagation both within and across node types by leveraging the typed edges in the heterogeneous subgraph.

\begin{figure}
    \centering
    \includegraphics[width=1\linewidth]{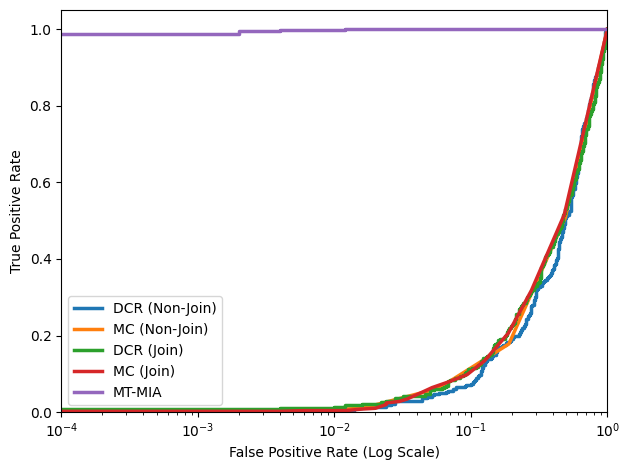}
    \caption{Plot of the True Positive Rate by log-scaled False Positive Rate for Section~\ref{subsec:motivexample}. Conventional single-table attacks Distance to Closest Record \cite{ganleaks} and MC \cite{Hilprecht2019MonteCA} cannot distinguish membership from the Customers table with or without additional Transaction rows joined as they cannot exploit the given inter-tabular leakage. MT-MIA learns and attacks a representation of the entire user subgraph allowing for nearly perfectly membership discrimination (AUC=0.999).}
    \Description{A line plot showing True Positive Rate on the y-axis (linear scale, 0 to 1) versus False Positive Rate on the x-axis (logarithmic scale, from 10 to the negative 4 to 1). Five attack methods are plotted as separate lines: DCR Non-Join, MC Non-Join, DCR Join, MC Join, and MT-MIA. The four single-table baselines (DCR Non-Join, MC Non-Join, DCR Join, MC Join) overlap closely along a diagonal line that rises slowly from low True Positive Rate at low False Positive Rate, indicating performance comparable to random guessing. The MT-MIA line sits at a True Positive Rate near 1.0 across the entire False Positive Rate range, including at the lowest False Positive Rates, indicating near-perfect membership discrimination.}
    \label{fig:motiv_ex_tpr}
\end{figure}

\paragraph{Dynamic Gated Fusion}
While the bifurcated signals provide a comprehensive view of the entity, the discriminative properties of these embeddings are unknown to the adversary. Privacy leakage may manifest in the unconditionally generated attributes ($z_{\text{parent}}$), the conditionally generated relational dependencies ($z_{\text{context}}$), or a latent intersection of both. To address this uncertainty, we employ a \textit{Dynamic Gating Unit} that adaptively modulates the integration of these signals.

We compute a learned gating vector $g \in [0,1]^d$ that serves as an entry wise modulator for relational influence. Given the parent and context embeddings, the gate is formulated as:
\begin{equation}
g = \sigma\!\left( \text{MLP}_{\text{gate}}([z_{\text{parent}} \parallel z_{\text{context}}]) \right),
\end{equation}
where $\sigma$ is the elementwise sigmoid activation and $\text{MLP}_{\text{gate}}$ is a small multilayer perceptron applied to the concatenation $[z_{\text{parent}} \parallel z_{\text{context}}]$. The final composite representation $z_{\text{final}}$ is then constructed via a gated residual connection:
\begin{equation}\label{eq:z_final}
z_{\text{final}} = z_{\text{parent}} + (g \odot \varphi(z_{\text{context}})),
\end{equation}
where $\varphi(\cdot)$ is a non linear transformation and $\odot$ denotes the Hadamard (elementwise) product. This mechanism allows $\mathcal{M}_\theta$ to ``tune'' the attack's sensitivity: it can prioritize intrinsic parent features when relational context is sparse, or amplify the structural signal when the generator exhibits strong conditional leakage. By allowing the data to determine the optimal weight of each signal, the encoder remains robust across various generative architectures and database schemas.

\begin{figure}
\centering
\adjustbox{max width=\columnwidth}{\begin{tikzpicture}[
scale=0.75,
transform shape,
node distance=1.2cm,
>=Stealth,
box/.style={draw, rounded corners=3pt, minimum width=1.5cm, minimum height=0.8cm, text centered, font=\scriptsize},
block/.style={box, fill=blue!5},
encoderA/.style={box, fill=red!20},
encoderB/.style={box, fill=blue!20},
encoderC/.style={box, fill=green!20},
arrow/.style={single arrow, single arrow head extend=0.15cm, draw, minimum height=0.8cm, minimum width=1cm, fill=gray!10},
pool/.style={circle, draw, minimum size=0.9cm, fill=purple!5, font=\tiny, align=center},
gate/.style={circle, draw, minimum size=1.1cm, fill=orange!10, align=center, font=\tiny},
arr/.style={->, thick},
att/.style={<->, dashed, thick, draw=gray!60},
nodeA/.style={circle, fill=red!60, minimum size=0.25cm},
nodeB/.style={circle, fill=blue!60, minimum size=0.25cm},
nodeC/.style={circle, fill=green!60, minimum size=0.25cm},
edge/.style={-, thick},
connection/.style={circle, fill=black, minimum size=0.2cm, inner sep=0}
]
\node[draw, rounded corners, minimum width=2.2cm, minimum height=1.6cm] (input) {};
\node[nodeA] (a1) at ($(input)+(-0.7,0.4)$) {};
\node[nodeA] (a2) at ($(input)+(-0.7,-0.0)$) {};
\node[nodeA] (a3) at ($(input)+(-0.7,-0.4)$) {};
\node[nodeB] (b1) at ($(input)+(0.0,0.4)$) {};
\node[nodeB] (b2) at ($(input)+(.0,-0.0)$) {};
\node[nodeC] (c1) at ($(input)+(.7,0.4)$) {};
\node[nodeC] (c2) at ($(input)+(.7,0)$) {};
\draw[edge] (a1) -- (b1);
\draw[edge] (a2) -- (b2);
\draw[edge] (a3) -- (b2);
\draw[edge] (b1) -- (c1);
\draw[edge] (b1) -- (c2);
\draw[edge] (b2) -- (c2);
\node[connection, right=1cm of input] (arr1) {};
\node[encoderA, right=1.8cm of arr1, yshift=2cm] (encA) {$\text{Enc}_A$};
\node[encoderB, right=1.8cm of arr1] (encB) {$\text{Enc}_B$};
\node[encoderC, right=1.8cm of arr1, yshift=-2cm] (encC) {$\text{Enc}_C$};
\draw[att] (encA) -- (encB);
\draw[att] (encB) -- (encC);
\draw[att] (encA) to[bend right=50] (encC);
\node[pool, right=1.2cm of encB, yshift=-1cm] (poolBC) {Attn\\Pool};
\node[gate, right=3.8cm of encB] (fusion) {Gated\\Fusion};
\node[right=0.8cm of fusion] (zfinal) {$z_{final}$};
\draw[arr] (input) -- (arr1);
\draw[arr] (arr1) |- (encA);
\draw[arr] (arr1) -- (encB);
\draw[arr] (arr1) |- (encC);
\draw[arr] (encA) -| node[pos=0.4, above, font=\small] {$z_{parent}$} (fusion);
\draw[arr] (encB) -- ($(encB)+(0.8,0)$) |- (poolBC);
\draw[arr] (encC) -- ($(encC)+(0.8,0)$) |- (poolBC);
\draw[arr] (poolBC) -| node[pos=0.4, below, font=\small, xshift=-0.2cm] {$z_{context}$} (fusion);
\draw[arr] (fusion) -- (zfinal);
\begin{pgfonlayer}{background}
\node[draw, thick, dashed, rounded corners, inner sep=10pt, fill=gray!3, fit=(encA) (encB) (encC)] {};
\end{pgfonlayer}
\node[above=0.1cm of input, font=\small] {$h^*$};
\node[above=0.5cm of encA, font=\small] {Message Passing HGNN};
\end{tikzpicture}}
\caption{Inference Time Architecture Diagram of MT-MIA. }
\label{fig:mt-mia-arch}
\Description{An architecture diagram showing the inference pipeline of MT-MIA. The input on the left is a candidate entity subgraph h-star, depicted as multiple connected nodes of different shapes representing different node types. The subgraph feeds into three parallel encoder blocks labeled Encoder A, Encoder B, and Encoder C, which together form the Message Passing HGNN stage. The encoders pass their outputs into an Attention Pooling block. The pooling block produces two intermediate embeddings: a parent embedding labeled z-parent and a context embedding labeled z-context. Both embeddings are passed into a Gated Fusion block, which combines them into the final output embedding labeled z-final.}
\end{figure}
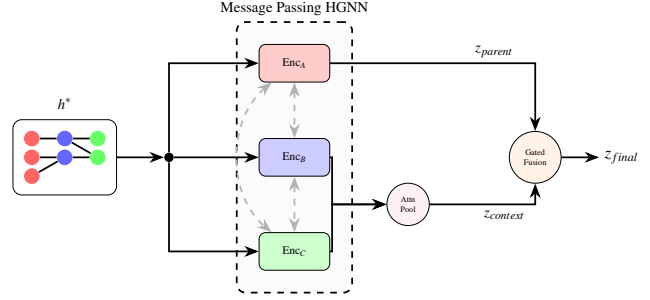
\subsection{Training in the No-Box Setting}
\label{sec:training}

In the No-Box setting, the adversary has access only to the synthetic output $G_{\text{synth}}$ without any knowledge of the generator's architecture, parameters, or training data $G_{\text{train}}$. Prior work in the single table setting \cite{ganleaks, Hilprecht2019MonteCA} establishes that synthetic data generators tend to leave detectable traces of memorization in their outputs: training records influence the synthetic distribution in ways that produce locally elevated density, near duplicate samples, or reduced reconstruction error around member records. Distance based attacks operationalize this observation by treating proximity between a candidate record and the synthetic dataset as evidence of membership. MT-MIA extends this principle to the relational setting: rather than measuring proximity in the raw feature space of a single row, we measure proximity in a learned embedding space that summarizes an entire entity subgraph, allowing the attack to detect memorization that manifests across connected rows.

\paragraph{Multi Anchor Reconstruction Objective}
Without access to ground truth membership labels, we train $\mathcal{M}_\theta$ using a self supervised reconstruction objective. The intuition is to force the encoder to learn a compressed representation that preserves both the primary entity's attributes and its relational neighborhood, which are the components that generators may inadvertently memorize during training.

We introduce two complementary reconstruction heads: a \textit{parent decoder} $\mathcal{D}_{\phi}$ with parameters $\phi$ that reconstructs the parent node's features, and a \textit{context decoder} $\mathcal{D}_{\psi}$ with parameters $\psi$ that reconstructs the aggregate neighborhood. For a target node $i \in \mathcal{V}_t$ with neighbor set $\mathcal{N}(i)$ spanning all adjacent node types, the composite reconstruction loss is:
\begin{equation}
\begin{aligned}
\mathcal{L}_\text{recon}(\theta, \phi, \psi) = \mathbb{E}_{i \sim \mathcal{V}_t} \bigg[ 
&\lambda_p \underbrace{\| \mathcal{D}_{\phi}(\mathcal{M}_\theta(x_i)) - \mathbf{x}_i^{(t)} \|^2_2}_{\text{Parent Recon.}} \\
+ &\lambda_c \underbrace{\Big\| \mathcal{D}_{\psi}(\mathcal{M}_\theta(x_i)) - \sum_{j \in \mathcal{N}(i)} \mathbf{x}_j \Big\|^2_2}_{\text{Context Recon.}} \bigg],
\end{aligned}
\end{equation}
where $\mathcal{M}_\theta(x_i)$ produces the fused embedding from Equation~\ref{eq:z_final} and $\lambda_p, \lambda_c \in \mathbb{R}_{\geq 0}$ are hyperparameters that balance the reconstruction of parent features against relational context. By learning to reconstruct both signals from a single bottleneck embedding, the encoder is forced to capture the specific relational motifs and conditional dependencies favored by the generator. Records exhibiting similar motifs produce embeddings that cluster near $G_{\text{synth}}$ in this learned space.

\paragraph{Membership Scoring}
While the learned embedding space can support various No-Box attacks, we derive our membership scoring from the Distance to Closest Record (DCR) attack \cite{ganleaks}. For a candidate entity subgraph $h^*$, we define $f(h^*)$ from Equation~\ref{eq:membership_prediction} as:
\begin{equation}
f(h^*) = - \min_{h \subseteq G_{\text{synth}}} \| \mathcal{M}_\theta(h^*) - \mathcal{M}_\theta(h) \|_2.
\end{equation}
Higher scores indicate that the structural motifs of $h^*$ were likely memorized and reproduced in $G_{\text{synth}}$.

\section{Experiments}
\label{sec:exp}
To evaluate the effectiveness of MT-MIA, we conduct a series of experiments on three benchmark multi-table datasets: California Census \cite{ipums_international_v7.3}, Airline Customers \cite{airline_loyalty_kaggle_impact}, and Airbnb \cite{airbnb_recruiting_new_user_bookings}. For each dataset, we begin by constructing training and holdout sets through the sampling of disjoint subgraphs, ensuring no overlap in entities or relationships. The synthetic data generator is then trained with the training subgraphs, after which we sample an equal number of synthetic subgraphs to match the original training size. All numeric features are scaled and categorical features are ordinally encoded in relation to the synthetic data, which are then applied consistently to both real and synthetic samples to prevent data leakage. 

We experiment with a training size of 1000 user subgraphs. These subgraphs correspond to thousands of items across all datasets' tables. To account for randomness in model training and sampling, each experimental configuration is repeated across three independent runs. Following the recommendations of prior work \cite{guépin2024lostaveragesnewspecific}, we fix the data split (training vs. holdout) across all runs and vary only the generative model initialization seeds. This design helps isolate the variability due to model behavior from that due to evaluation set construction, which is especially important in privacy attack scenarios.

Following \cite{vanbreugel2023membership}, \cite{Synth_MIA} and \cite{ward2024dataplagiarismindexcharacterizing}, all training data are included as the positive membership class with an equal sized holdout dataset as a negative class. All MIAs are then evaluated with the corresponding synthetic data on this evaluation set to then calculate the success of the attack. 

We run all experiments on a High Performance Computing Cluster using a Nvidia H200 GPU with a 32 core CPU. The full synthetic data generation procedure for the models RealTabFormer and ClavaDDPM was approximately 10 hours of compute on this system. For the much larger and more expensive RelDiff, the procedure was approximately 450 GPU hours. The MT-MIA training and inference procedure was approximately 1 hour of compute time over all runs. Additional compute was used for preliminary experiments.

\subsection{Baselines} 
\subsubsection{Multi-Table Synthetic Data Generators}
We evaluate our proposed approach against three representative multi relational generative models that span autoregressive sequence modeling, hierarchical diffusion, and graph structured diffusion approaches. 
\begin{itemize}
    \item \textbf{RealTabFormer \cite{solatorio2023realtabformer}:} This model synthesizes multi-relational data by framing child table generation as a conditional sequence generation task. It utilizes a GPT-based architecture where parent records are encoded to form a context for a sequence-to-sequence (Seq2Seq) model. The generator produces child rows while maintaining one-to-many relationship cardinality by treating primary-key-foreign-key links as causal sequences. This approach models conditional distributions without requiring manual schema flattening.
    
    \item \textbf{ClavaDDPM \cite{clavaddpm}:} This framework generates multi-relational data through a hierarchical guidance mechanism using cluster-based latent variables. The model applies a clustering algorithm to the parent table to extract latent representations, which serve as conditioning signals for the diffusion process of the associated child tables. During the reverse denoising step, the model optimizes a conditional objective function to align child records with parent clusters. This structure captures dependencies across the database schema without the computational overhead of autoregressive or graph-based methods.

    \item \textbf{RelDiff \cite{hudovernik2025reldiffrelationaldatagenerative}:} This framework synthesizes complete relational databases by explicitly modeling their Foreign Key graph structure. RelDiff decomposes generation into two stages: a joint graph conditioned diffusion process that synthesizes attributes across all tables simultaneously, and a Stochastic Block Model based graph generator that synthesizes the Foreign Key structure itself. This decomposition of graph structure from relational attributes is designed to preserve both fidelity and referential integrity, avoiding the structural assumptions imposed by methods that flatten relational data into conditionally generated tables.    
\end{itemize}
\begin{table*}[t]
  \centering
  \small
  \setlength{\tabcolsep}{4pt}
  \caption{Per-(model, dataset) comparison of MT-MIA against the best baseline attack. All values are means across seeds. For each pair, the larger value is \textbf{bold}; $\Delta$ reports the absolute gain of MT-MIA over the baseline. A table with standard deviations are reported in Appendix~\ref{app:add-results}.}
  \label{tab:mt-mia-vs-baseline}
  \begin{tabular}{llccccccccc}
    \toprule
    Model & Metric & \multicolumn{3}{c}{California} & \multicolumn{3}{c}{Airbnb} & \multicolumn{3}{c}{Airlines} \\
    \cmidrule(lr){3-5} \cmidrule(lr){6-8} \cmidrule(lr){9-11}
     &  & Baseline & MT-MIA & $\Delta$ & Baseline & MT-MIA & $\Delta$ & Baseline & MT-MIA & $\Delta$ \\
    \midrule
    \textsc{ClavaDDPM} & AUC-ROC & \textbf{0.79} & 0.69 & \cellcolor{LossStrong}$-$0.10 & 0.79 & \textbf{0.80} & \cellcolor{GainLight}+0.01 & \textbf{0.69} & 0.66 & \cellcolor{LossLight}$-$0.03 \\
     & TPR@FPR$=$0 & 0.00 & \textbf{0.07} & \cellcolor{GainMed}+0.07 & 0.00 & \textbf{0.01} & \cellcolor{GainLight}+0.01 & 0.07 & \textbf{0.17} & \cellcolor{GainStrong}+0.10 \\
     & TPR@FPR$=$$10^{-3}$ & 0.01 & \textbf{0.09} & \cellcolor{GainMed}+0.08 & 0.01 & 0.01 & 0.00 & 0.08 & \textbf{0.21} & \cellcolor{GainStrong}+0.13 \\
     & TPR@FPR$=$$10^{-2}$ & 0.11 & \textbf{0.15} & \cellcolor{GainLight}+0.04 & 0.06 & \textbf{0.09} & \cellcolor{GainLight}+0.03 & 0.12 & \textbf{0.31} & \cellcolor{GainStrong}+0.19 \\
    \midrule
    \textsc{RelDiff} & AUC-ROC & \textbf{0.67} & 0.64 & \cellcolor{LossLight}$-$0.03 & 0.57 & \textbf{0.62} & \cellcolor{GainMed}+0.05 & \textbf{0.51} & 0.49 & \cellcolor{LossLight}$-$0.02 \\
     & TPR@FPR$=$0 & 0.00 & \textbf{0.15} & \cellcolor{GainStrong}+0.15 & 0.00 & \textbf{0.03} & \cellcolor{GainLight}+0.03 & 0.00 & 0.00 & 0.00 \\
     & TPR@FPR$=$$10^{-3}$ & 0.01 & \textbf{0.17} & \cellcolor{GainStrong}+0.16 & 0.00 & \textbf{0.03} & \cellcolor{GainLight}+0.03 & 0.00 & 0.00 & 0.00 \\
     & TPR@FPR$=$$10^{-2}$ & 0.13 & \textbf{0.22} & \cellcolor{GainMed}+0.09 & 0.02 & \textbf{0.06} & \cellcolor{GainLight}+0.04 & 0.01 & 0.01 & 0.00 \\
    \midrule
    \textsc{RTF} & AUC-ROC & \textbf{0.57} & 0.52 & \cellcolor{LossMed}$-$0.05 & 0.58 & 0.58 & 0.00 & \textbf{0.54} & 0.51 & \cellcolor{LossLight}$-$0.03 \\
     & TPR@FPR$=$0 & 0.00 & 0.00 & 0.00 & 0.00 & 0.00 & 0.00 & 0.01 & 0.01 & 0.00 \\
     & TPR@FPR$=$$10^{-3}$ & 0.00 & 0.00 & 0.00 & \textbf{0.01} & 0.00 & \cellcolor{LossLight}$-$0.01 & \textbf{0.02} & 0.01 & \cellcolor{LossLight}$-$0.01 \\
     & TPR@FPR$=$$10^{-2}$ & 0.02 & 0.02 & 0.00 & 0.02 & 0.02 & 0.00 & \textbf{0.05} & 0.03 & \cellcolor{LossLight}$-$0.02 \\
    \bottomrule
  \end{tabular}
\end{table*}

\subsubsection{Attacks}
Since prior work has not investigated MIAs in the context of synthetic database release, we adapt existing single-table attack methods that align with our threat model for benchmarking purposes. In our setting, membership inference on a user subgraph $h^*$ is equivalent to determining whether all nodes $v \in h^*$ were present in the training data. Per Theorem~\ref{thm:leakage}, this can be reframed as selecting 1 node for each subgraph in which to use a single-table attack on. We therefore select a parent node in each $h^*$ (the strategy of Figure~\ref{fig:motive_a}) and use its corresponding score as the attack score for $h^*$. 

We use three common No-box attacks for tabular data as these scoring functions: Distance to Closest Record (DCR) \cite{ganleaks}, a Monte Carlo Density Estimator (MC) proposed by \cite{Hilprecht2019MonteCA}, and a Kernel Density Estimator (KDE) approach used in \cite{houssiau2022tapas}.
\subsection{Metrics}
\label{sec:metrics}

\subsubsection{Attack Success}
\label{sec:attack_success}

We evaluate attack effectiveness following the framework established by \cite{Carlini2021MembershipIA}, which argues that membership inference attacks should be assessed by their behavior in the high confidence regime rather than by aggregate classification metrics. Our primary metric is the True Positive Rate at low False Positive Rates (TPR@FPR), reported at FPR levels of $0$, $10^{-3}$, and $10^{-2}$. We also report the Area Under the ROC Curve (AUC) for completeness and comparability with prior work, but treat it as a secondary measure.

A high TPR at a low FPR is the most relevant regime for privacy auditing. An adversary that can confidently identify even a small fraction of training records with few false positives poses a real privacy risk, while an adversary that must accept many false positives to flag the same number of true members cannot reliably attribute leakage to any specific record. Attacks that only separate members and non members on average, across the full operating curve, do not produce this kind of high confidence identification.
\subsubsection{Synthetic Data Fidelity}
\label{sec:fidelity}

We evaluate the quality of the multi table synthetic data using four metrics : (1) \textit{cardinality}, which measures the Foreign Key group size distribution to assess intra group correlations; (2) \textit{column wise density estimation (1 way)}, which estimates the marginal density of every column across all tables; (3) \textit{pairwise column correlation ($k$ hop)}, which assesses dependencies between columns at distance $k$ (for example, 0 hop for intra table and 1 hop for parent child relations); and (4) \textit{average $k$ hop}, which averages the $k$ hop correlation scores over $k \in \{0, 1, \dots, K\}$ for a schema of maximum join depth $K$, summarizing both intra table ($k = 0$) and cross table ($k \geq 1$) dependencies in a single fidelity score. For each measure following \cite{clavaddpm}, we report the complement of the Kolmogorov Smirnov (KS) statistic and Total Variation (TV) distance, normalized to $[0, 1]$ where 1 indicates perfect fidelity.

\section{Discussion}
\subsection{Privacy Auditing Multi-Table Synthetic Data}

\subsubsection{Performance of MIAs}
\label{sec:mia_performance}

MT-MIA's central advantage is that it scores membership over the full entity subgraph rather than a single row, which allows it to surface privacy violations that single table attacks are structurally incapable of detecting. The clearest evidence for this advantage appears in the high confidence regime that Section~\ref{sec:attack_success} identifies as the privacy relevant operating point of an MIA. Table~\ref{tab:mt-mia-vs-baseline} compares the mean best single table attack for each metric against MT-MIA. We find that MT-MIA delivers consistent and often substantial gains in TPR at low FPR across both ClavaDDPM and RelDiff.

The most striking pattern across our results is that MT-MIA reveals leakage where baselines detect none at all. On RelDiff with California, the strongest single table baseline achieves TPR@FPR$=$0 of 0.00, while MT-MIA reaches 0.15, a 15 percentage point gap that corresponds to confidently identifying roughly 150 training records with zero false positives in our evaluation. The same pattern holds on ClavaDDPM with California (0.00 to 0.07), ClavaDDPM with Airbnb (0.00 to 0.01), and ClavaDDPM with Airlines, where MT-MIA improves on the strongest baseline by 10 percentage points at FPR$=$0 and by 19 points at FPR$=$$10^{-2}$. In each of these settings, a practitioner relying on a single table audit would conclude that no high confidence privacy leakage is present; MT-MIA shows that this conclusion is wrong.

A notable pattern in these results is that incorporating relational data does not always improve AUC, yet consistently improves calibration in the high confidence regime. On several configurations, single table attacks achieve comparable or slightly higher AUC than MT-MIA while MT-MIA still detects strictly more leakage at low FPR. This suggests that the relational signal sharpens the high confidence end of the score distribution rather than uniformly separating members from non members.

A consistent property of MT-MIA across all three generators is that its scores track the inter-table signal each generator preserves. RealTabFormer, whose synthetic outputs lose substantial inter-table structure (Section~\ref{sec:privacy_fidelity}), produces small MT-MIA gains over single table baselines because the relational motifs MT-MIA targets are largely absent from $G_{\text{synth}}$. ClavaDDPM and RelDiff, which preserve relational structure faithfully, are where MT-MIA surfaces large gains. The attack is therefore well calibrated to the property it is designed to detect: it returns strong signal where inter-table leakage exists in the synthetic output and stays muted where it does not.
\begin{table*}[t]
  \centering
  \setlength{\tabcolsep}{5pt}
  \caption{Fidelity metrics (left) versus MT-MIA privacy leakage (right) for each (model, dataset). Values are $\text{mean}_{\pm\sigma}$ over three seeds. Higher fidelity values indicate greater quality synthetic data; higher MT-MIA values indicate greater privacy leakage. \textsc{RTF} has substantially worse fidelity but this also translates to less privacy leakage, while \textsc{ClavaDDPM} and \textsc{RelDiff} achieve strong fidelity and leak substantially under MT-MIA.}
  \label{tab:fid_results}

  \begin{tabular}{llccc@{\hspace{1.5em}}cc}
    \toprule
    Model & Dataset & \multicolumn{3}{c}{Fidelity} & \multicolumn{2}{c}{MT-MIA} \\
    \cmidrule(lr){3-5} \cmidrule(lr){6-7}
     &  & Col.~Shapes & Cardinality & Avg.~$k$-hop & AUC-ROC & TPR@FPR$=$$10^{-3}$ \\
    \midrule
    \textsc{ClavaDDPM} & California & $0.97_{\,\pm0.00}$ & $0.98_{\,\pm0.01}$ & $0.93_{\,\pm0.01}$ & $0.69_{\,\pm0.05}$ & $0.09_{\,\pm0.08}$ \\
     & Airbnb & $0.98_{\,\pm0.00}$ & $0.98_{\,\pm0.01}$ & $0.92_{\,\pm0.03}$ & $0.80_{\,\pm0.01}$ & $0.01_{\,\pm0.00}$ \\
     & Airlines & $0.99_{\,\pm0.00}$ & $0.99_{\,\pm0.00}$ & $0.97_{\,\pm0.00}$ & $0.66_{\,\pm0.01}$ & $0.21_{\,\pm0.06}$ \\
    \midrule
    \textsc{RelDiff} & California & $0.97_{\,\pm0.00}$ & $1.00_{\,\pm0.00}$ & $0.93_{\,\pm0.00}$ & $0.64_{\,\pm0.01}$ & $0.17_{\,\pm0.01}$ \\
     & Airbnb & $0.98_{\,\pm0.00}$ & $1.00_{\,\pm0.00}$ & $0.90_{\,\pm0.03}$ & $0.62_{\,\pm0.01}$ & $0.03_{\,\pm0.02}$ \\
     & Airlines & $0.95_{\,\pm0.01}$ & $1.00_{\,\pm0.00}$ & $0.93_{\,\pm0.01}$ & $0.49_{\,\pm0.01}$ & $0.00_{\,\pm0.00}$ \\
    \midrule
    \textsc{RTF} & California & $0.84_{\,\pm0.01}$ & $0.78_{\,\pm0.13}$ & $0.72_{\,\pm0.01}$ & $0.52_{\,\pm0.02}$ & $0.00_{\,\pm0.00}$ \\
     & Airbnb & $0.90_{\,\pm0.02}$ & $0.85_{\,\pm0.01}$ & $0.75_{\,\pm0.04}$ & $0.58_{\,\pm0.00}$ & $0.00_{\,\pm0.00}$ \\
     & Airlines & $0.87_{\,\pm0.01}$ & $0.06_{\,\pm0.00}$ & $0.80_{\,\pm0.01}$ & $0.51_{\,\pm0.01}$ & $0.01_{\,\pm0.01}$ \\
    \bottomrule
  \end{tabular}
\end{table*}

\subsubsection{Privacy versus Fidelity Tradeoff}
\label{sec:privacy_fidelity}

 We compare synthetic data fidelity against MT-MIA's detected leakage across all three generators in Table~\ref{tab:fid_results}. ClavaDDPM and RelDiff achieve uniformly strong fidelity, with column shape, cardinality, and average $k$ hop scores at or above 0.90 in nearly every cell. RealTabFormer's fidelity is substantially weaker for all metrics and datasets, particularly on Airlines where its cardinality score collapses to 0.06.

This fidelity gap maps directly onto leakage. Both high fidelity generators leak under MT-MIA: ClavaDDPM reaches TPR@FPR$=$$10^{-3}$ of 0.21 on Airlines, and RelDiff reaches 0.17 on California. RealTabFormer leaks substantially less, with TPR@FPR$=$$10^{-3}$ at or below 0.01 across all configurations. The relationship between fidelity and leakage is consistent with findings in the single table synthetic data literature: stronger preservation of intra and inter-table dependencies is correlated with greater privacy vulnerability \cite{meeatchi24, Synth_MIA}, and diffusion based generators in particular have been observed to be more susceptible to memorization based attacks than alternative architectures. Our results indicate that this relationship extends to the multi table setting and applies to both diffusion based architectures we evaluate.

The interpretation of RealTabFormer's low MT-MIA scores warrants care. Table~\ref{tab:mt-mia-vs-baseline} shows that single table baselines do detect parent table leakage on RealTabFormer (AUC$=$0.57 on California, 0.58 on Airbnb). What MT-MIA finds less of in RealTabFormer's outputs is faithful inter-table structure: with cardinality and average $k$ hop scores substantially below ClavaDDPM and RelDiff, the relational motifs MT-MIA scores against are likely not present in $G_{\text{synth}}$ for the attack to exploit. MT-MIA's effectiveness is therefore tied to the strength of the inter-table signal a generator preserves: the more faithfully a generator reproduces relational structure, the more calibrated MT-MIA becomes relative to single table baselines.

\begin{table}
\small
\centering
\caption{MIA metrics for intermediate and final embeddings found in MT-MIA for ClavaDDPM. We conduct DCR attacks on the single-table setting (Vanilla), the intermediate embeddings of MT-MIA $z_{parent}$ and $z_{context}$, and the final embeddings in the attack $z_{final}$. ClavaDDPM experiences privacy leakage for different datasets in different components of multi-table synthesis.}

\begin{tabular}{ll|cccc}

\hline
{Dataset} & {Metric} & {Vanilla} & {$z_{parent}$} & {$z_{context}$} & {$z_{final}$} \\
\hline
\multirow{4}*{California} & AUC & \textbf{0.781} & 0.768 & 0.650 & 0.746 \\
 & TPR@FPR=0 & 0.000 & 0.000 & \textbf{0.183} & 0.028 \\
 & TPR@FPR=$10^{-3}$ & 0.009 & 0.007 & \textbf{0.282} & 0.051 \\
 & TPR@FPR=$10^{-2}$ & 0.093 & 0.072 & \textbf{0.315} & 0.150 \\
\hline
\multirow{4}*{Airbnb} & AUC & 0.781 & 0.793 & 0.533 & \textbf{0.795} \\
 & TPR@FPR=0 & 0.000 & 0.000 & 0.000 & \textbf{0.002} \\
 & TPR@FPR=$10^{-3}$ & 0.006 & 0.006 & 0.001 & \textbf{0.009} \\
 & TPR@FPR=$10^{-2}$ & 0.056 & 0.061 & 0.013 & \textbf{0.094} \\
\hline
\multirow{4}*{Airlines} & AUC & 0.690 & \textbf{0.713} & 0.503 & 0.660 \\
 & TPR@FPR=0 & 0.055 & \textbf{0.356} & 0.002 & 0.155 \\
 & TPR@FPR=$10^{-3}$ & 0.089 & \textbf{0.365} & 0.007 & 0.245 \\
 & TPR@FPR=$10^{-2}$ & 0.122 & \textbf{0.411} & 0.014 & 0.308 \\
\hline

\end{tabular}
\label{tab:mtmia-decomp}
\end{table}
\subsection{Sources of Leakage: Node Versus Neighborhood}
\label{sec:sources}

MT-MIA's performance gains stem from its ability to expose distinct sources of privacy leakage that are inaccessible to single table attacks. The HGNN backbone produces three intermediate representations of an entity subgraph: the parent embedding $z_{\text{parent}}$, the relational context embedding $z_{\text{context}}$, and the fused embedding $z_{\text{final}}$. This decomposition allows us to probe which components of the relational structure contribute to membership distinguishability and to attribute MT-MIA's gains to specific leakage pathways rather than to increased model capacity alone.

To quantify the contribution of each component, we apply the same DCR attack independently to each embedding space and compare against a Vanilla single table DCR attack on the original parent table feature space. We report attack performance for the best runs on each dataset under ClavaDDPM in Table~\ref{tab:mtmia-decomp}.

On California, attacking $z_{\text{parent}}$ yields AUC and TPR@FPR values nearly identical to the Vanilla attack, indicating that parent level attributes alone do not provide a substantially stronger signal once embedded. In contrast, $z_{\text{context}}$ reveals a markedly stronger signal, with improvements of 18 percentage points at TPR@FPR$=$0 and 28 percentage points at TPR@FPR$=$$10^{-3}$ over Vanilla. This suggests that ClavaDDPM memorizes recurring relational motifs in child records, which remain invisible to single table attacks but become exploitable once relational neighborhoods are explicitly encoded.

The pattern inverts on Airlines and Airbnb, where the primary gains arise from $z_{\text{parent}}$ rather than $z_{\text{context}}$. On Airlines, $z_{\text{parent}}$ achieves a 30 percentage point increase in TPR@FPR$=$0 over Vanilla while $z_{\text{context}}$ uncovers little to no membership signal. On Airbnb the effect is more modest: $z_{\text{parent}}$ and $z_{\text{final}}$ both outperform Vanilla while $z_{\text{context}}$ alone provides little. We attribute these gains to the message passing mechanism of the HGNN, which aggregates information across relational edges during embedding construction; even when neighborhood nodes do not themselves encode a strong membership signal, message passing can amplify subtle parent level differences and reshape the geometry of the representation space.
Across all three datasets, attacking either $z_{\text{parent}}$ or $z_{\text{context}}$ in isolation can yield stronger attack performance than attacking $z_{\text{final}}$. However, under our threat model an adversary does not have a priori knowledge of which component contains the dominant membership signal for a given dataset or generative model. MT-MIA therefore relies on $z_{\text{final}}$ as a robust, model agnostic attack strategy that does not require such prior assumptions.

The availability of separate embedding channels provides a useful diagnostic for internal auditing and model development. By independently probing $z_{\text{parent}}$ and $z_{\text{context}}$, practitioners can identify whether privacy leakage primarily arises from memorization of parent attributes or from recurring relational motifs in child tables. This decomposition enables targeted mitigation strategies, such as regularizing parent representations when leakage concentrates in $z_{\text{parent}}$ or modifying relational modeling or sampling procedures when leakage is driven by neighborhood structure.

\begin{table*}[t]
\centering
\small
\caption{Child table attack comparison for ClavaDDPM (Mean $\pm$ Std). We deploy a single-table DCR against the children tables of each dataset and compare it to corresponding scores from MT-MIA targeting user graphs, exhibiting the ''weakest link'' effect implied by multi-table synthetic data release.}
\label{tab:child_attacks}
\begin{tabular}{l|cccc|cccc}
\toprule
\multirow{2}{*}{Dataset} & \multicolumn{4}{c|}{Child Table DCR} & \multicolumn{4}{c}{Child Table MT-MIA} \\
\cmidrule(lr){2-5} \cmidrule(lr){6-9}
 & AUC & TPR@FPR0 & TPR@FPR$10^{-3}$ & TPR@FPR$10^{-2}$ 
 & AUC & TPR@FPR0 & TPR@FPR$10^{-3}$ & TPR@FPR$10^{-2}$ \\
\midrule
Airbnb     & 0.51$ \pm $0.00 & 0.00$ \pm $0.00 & 0.00$ \pm $0.00 & 0.01$ \pm $0.00 
           & \textbf{0.79}$ \pm ${0.01} & 0.00$ \pm ${0.01} & 0.00$ \pm $0.01 & \textbf{0.09}$ \pm $0.01 \\
Airlines   & 0.51$ \pm $0.00 & 0.00$ \pm $0.00 & 0.00$ \pm $0.00 & 0.01$ \pm $0.00
           & \textbf{0.66}$ \pm ${0.02} & \textbf{0.12}$ \pm ${0.08} & 0.12$ \pm $0.08 & \textbf{0.21}$ \pm $0.13 \\
California & \textbf{0.75}$ \pm $0.00 & 0.00$ \pm $0.00 & 0.00$ \pm $0.00 & 0.04$ \pm $0.00
           & 0.69$ \pm ${0.05} & \textbf{0.04}$ \pm $0.04 & \textbf{0.06}$ \pm $0.05 & \textbf{0.13}$ \pm $0.11 \\
\bottomrule
\end{tabular}
\end{table*}
\subsection{Implied Privacy of User Items}
\label{sec:implied_privacy}
While the prior subsections audit privacy at the entity level, Theorem~\ref{thm:leakage} also implies a result about item level privacy: an item's privacy is bounded by its least private representation across all connected tables in the relational schema. To empirically validate this ``weakest link'' effect, we compare MT-MIA against single table Distance to Closest Record (DCR) attacks applied in isolation to the child tables of each dataset for ClavaDDPM in Table~\ref{tab:child_attacks}.
With the exception of California, single table DCR attacks detect little privacy leakage when applied to child tables. This is consistent with the i.i.d. assumption underlying these attacks, which treats each child observation as an independent sample and ignores the structural constraints imposed by the parent child join relationships. MT-MIA, by contrast, utilizes information for the entire subgraph and treats membership inference on child items as inference on the overall entity. On Airlines and California, MT-MIA substantially outperforms the single table baseline in the TPR@FPR regime, recovering leakage that the localized audit misses entirely. The implication is operationally important for relational synthetic data release: a user item that appears safe under a single table audit can leak privacy when its considered together with its neighborhood of parents and siblings.

\subsection{Limitations}
\label{sec:limitations}

Despite the efficacy of MT-MIA, several limitations warrant further discussion.

\paragraph{Threat Model Constraints} MT-MIA operates under a No-Box threat model in which the adversary possesses only the synthetic output $G_{\text{synth}}$ and the database schema. As discussed in Section~\ref{sec:threat_model}, this is the threat model that most closely matches how synthetic relational data is released in practice and is what makes MT-MIA model agnostic and dataset agnostic. The broader lesson of MT-MIA, however, is that learning an embedding of the full entity subgraph is what surfaces inter-table leakage, and this lesson should extend to less conservative threat models: a Calibrated No-Box or Shadow-Box attack that operates over learned subgraph embeddings rather than single rows would inherit the same advantage MT-MIA demonstrates over single table baselines. We see adapting these stronger attacks to the relational setting through learned representations as a natural direction for future work.

\paragraph{Disjoint Entity Subgraphs} Theorem~\ref{thm:leakage} and the auditing setup in Section~\ref{sec:auditing_subgraphs} assume that user entities decompose into disjoint connected subgraphs, which holds for schemas where each user's data is fully separable from every other user's. Schemas with shared reference tables, such as a \texttt{Products} table referenced by every user's transactions, or schemas with many to many relationships induce overlapping entity subgraphs in which the disjointness assumption does not hold cleanly. The MT-MIA encoder itself does not require disjointness, since the HGNN operates on whatever subgraph it is given; what changes is the auditing semantics, as a shared node in a member subgraph is necessarily also a node in some non member subgraph. Defining the appropriate unit of privacy under shared references and many to many joins is a modeling choice that depends on the auditor's goals, and we leave a systematic study of subgraph definition under these schemas to future work.

\paragraph{Sensitivity to Embedding Quality} As a representation learning based attack, the success of MT-MIA is intrinsically tied to the discriminative power of the latent space learned by the HGNN. Graph Neural Networks are known to be sensitive to structural noise and hyperparameter configurations such as learning rate, message passing depth, and pooling strategies \cite{10.1145/3366423.3380297, fey2023relationaldeeplearninggraph, yang2023simpleefficientheterogeneousgraph}. Our experiments suggest that MT-MIA remains robust across the schemas we evaluate, but this sensitivity warrants attention when applying the attack to new domains.

\paragraph{Signal Integration and Gated Fusion} The Dynamic Gating Unit adaptively weights $z_{\text{parent}}$ and $z_{\text{context}}$ but does not always yield a composite embedding that is more discriminative than the individual signals. As discussed in Section~\ref{sec:sources}, the raw parent or context vectors independently achieve stronger attack performance than the fused $z_{\text{final}}$ on several configurations. Under our threat model an adversary does not know in advance which channel will dominate for a given generator and dataset, so MT-MIA defaults to $z_{\text{final}}$ as a robust strategy, but a more principled fusion mechanism that reliably matches or exceeds the best individual channel is a direction for future work.

\section{Conclusion}
We present the first systematic study of user-level privacy auditing for synthetic relational data generation, demonstrating both theoretically and empirically that multi-table settings introduce privacy leakage at a user-level. Our proposed Multi-Table Membership Inference Attack (MT-MIA) leverages heterogeneous graph neural networks in a self-supervised manner to detect membership information leakage across connected entities without requiring generator access. Evaluation across multiple real-world datasets shows MT-MIA consistently improves upon existing single-table approaches, particularly in the critical low false-positive regime, revealing that state-of-the-art relational generators leak membership information under conservative threat models.

There are many directions for future work in this area. Efforts could focus on refining HGNN architectures to improve the fidelity of learned subgraph representations, which would likely enhance attack performance. Extending the embedding based approach of MT-MIA to less conservative threat models would also be valuable, particularly if specific architectures become popular for relational data generation. Additionally, developing user-level differentially private relational data generators would likely be valuable in protecting user privacy. Finally, this work opens up additional lines of inquiry in extending other privacy auditing paradigms such as link prediction and attribute inference to the relational data setting.

\section{Ethical Considerations}

This work develops an attack against released synthetic relational data and demonstrates that existing generators leak membership information at the user level. We consider the ethical implications across the relevant stakeholders.

\textbf{Stakeholders.} The primary stakeholders are individuals whose records appear in relational databases that may be released as synthetic data, particularly in domains where membership itself is sensitive (healthcare records, financial transactions, social services interactions). Secondary stakeholders include data curators who release synthetic relational data and assume that synthesis is sufficient to protect contributors, researchers developing synthetic relational data generators, and practitioners conducting privacy audits.

\textbf{Impact of the research process.} MT-MIA was developed and evaluated using publicly available datasets that were released for research purposes. No additional individual data was collected, and no real synthetic data deployments were attacked.

\textbf{Impact of publication.} Publishing MT-MIA presents a tension. The attack could in principle be used by an adversary against a released synthetic relational dataset, particularly one generated by ClavaDDPM or RelDiff, the generators we evaluate. However, the alternative of not publishing leaves data curators with no auditing tool for user level privacy in the relational setting and no awareness that single table audits underestimate the privacy risk of multi table releases. Our judgment is that the population of curators who would benefit from understanding this risk substantially exceeds the marginal capability publication provides to adversaries, who can already attempt single table attacks. We further argue that adversarial auditing is a precondition for the development of user level differentially private relational generators (Section~\ref{sec:limitations}), which would address the underlying vulnerability.

\textbf{Decision to publish.} We considered whether to disclose this vulnerability privately to authors of the evaluated generators before publication. We decided against this for two reasons: first, the vulnerability is structural rather than implementation specific, so no patch is available for individual generators; second, the affected population (data curators considering relational synthetic data release) is broad and not tied to any single project, making coordinated disclosure infeasible. We instead release MT-MIA as an auditing tool alongside the paper.


\bibliographystyle{ACM-Reference-Format}
\bibliography{main}

\appendix

\section{Proofs}
\label{app:proofs}
\setcounter{theorem}{0}
\begin{theorem}
Let $G_{\text{test}} = (V, E)$ be a graph that is the disjoint union of two subgraphs $G_{\text{member}}$ and $G_{\text{holdout}}$, where $V(G) = V(G_{\text{member}}) \cup V(G_{\text{holdout}})$ and $V(G_{\text{member}}) \cap V(G_{\text{holdout}}) = \emptyset$. Furthermore, there are no edges in $G$ connecting vertices between $G_{\text{member}}$ and $G_{\text{holdout}}$.
Let $h^* \subseteq G$ be a connected subgraph, and let $g \subseteq h^*$.

Then, if $g \subseteq G_{\text{member}}$, it follows that $h^* \subseteq G_{\text{member}}$. Likewise, if $g \subseteq G_{\text{holdout}}$, then $h^* \subseteq G_{\text{holdout}}$.
\end{theorem}

\begin{proof}
We prove the first statement; the second follows by symmetry.

Assume $g \subseteq G_{\text{member}}$ and suppose, for contradiction, that $h^* \not\subseteq G_{\text{member}}$. Then there exists at least one vertex $v \in V(h^*)$ such that $v \in V(G_{\text{holdout}})$.

Since $g \subseteq h^*$ and $g \subseteq G_{\text{member}}$, there exists at least one vertex $u \in V(g) \subseteq V(G_{\text{member}})$.

Since $h^*$ is connected, there must exist a path $P$ in $h^*$ from $u$ to $v$. As $u \in V(G_{\text{member}})$ and $v \in V(G_{\text{holdout}})$, this path must contain an edge $(v_i, v_{i+1})$ where $v_i \in V(G_{\text{member}})$ and $v_{i+1} \in V(G_{\text{holdout}})$.

However, by assumption, no edges exist in $G$ between vertices of $G_{\text{member}}$ and $G_{\text{holdout}}$. This contradiction proves that $h^* \subseteq G_{\text{member}}$.

The second statement follows by an identical argument with the roles of $G_{\text{member}}$ and $G_{\text{holdout}}$ reversed.
\end{proof}
\section{Notation}
\label{app:notation}

Table~\ref{tab:notation} summarizes the symbols introduced in Section~\ref{sec:methodology}.

\section{Experiments/ Reproducibility}

\subsection{MT-MIA Model Components}
The MT-MIA architecture consists of four distinct functional stages designed to balance local record features with global relational structure:

\begin{itemize}
    \item \textbf{Heterogeneous Message Passing (Encoder):} The core of the model utilizes two stacks of \texttt{GNNEncoder} layers, transformed via the \texttt{to\_hetero} utility. Each stack employs \textbf{GATv2} (Graph Attention Network v2) layers with a hidden dimension $d = 1024$, allowing the model to learn type-specific relationships across the relational schema.
    
    \item \textbf{Target and Context Isolation:} 
    \begin{itemize}
        \item \textit{Target Signal ($z_{\text{target}}$):} Node embeddings for the {target node type} (the record being audited) are isolated and processed via \texttt{LayerNorm}.
        \item \textit{Context Signal ($z_{\text{context}}$):} An attention-based global pooling mechanism aggregates features from all non-target node types, representing the "relational neighborhood" of the target record.
    \end{itemize}

    \item \textbf{Gated Fusion Mechanism:} To prevent over-reliance on either the target record or its context, a gating unit calculates a scalar $g \in [0, 1]$ via a sigmoid activation:
    \begin{equation}
        g = \sigma(W_{\text{gate}}[z_{\text{target}} \parallel z_{\text{context}}] + b_{\text{gate}})
    \end{equation}
    The final representation is computed as a gated residual connection:
    \begin{equation}
        z_{\text{final}} = z_{\text{target}} + (g \odot \text{Transform}(z_{\text{context}}))
    \end{equation}

    \item \textbf{Structural Anchors (Decoders):} Reconstruction heads for both target and context map hidden representations back to original feature dimensions. This acts as a structural anchor, ensuring the latent space preserves the physical characteristics of the data.
\end{itemize}

\subsection{Hyperparameter Configuration}
Table~\ref{tab:hyperparams} summarizes the primary architectural parameters used in the MT-MIA implementation.

\begin{table}[h]
\small
\centering
\caption{MT-MIA Architectural Hyperparameters}
\label{tab:hyperparams}
\resizebox{.5\textwidth}{!}{
\begin{tabular}{lll}
\hline
\textbf{Parameter} & \textbf{Value} & \textbf{Description} \\ \hline
\texttt{hidden\_channels} & 1024 & Dimensionality of the latent space. \\
\texttt{num\_conv\_stacks} & 2 & Number of het. message-passing blocks. \\
\texttt{conv\_operator} & GATv2 & Graph Attention Network v2 operator. \\
\texttt{aggregation} & Attn Pool & Weights and sums node embeddings. \\
\texttt{activation} & ReLU & Non-linear activation function. \\
\texttt{gate\_activation} & Sigmoid & Actv. for the context influence gate. \\ \hline
\end{tabular}%
}
\end{table}

\begin{table*}
\centering
\caption{Dataset links, configurations and sample sizes across experimental runs.}
\label{tab:dataset_configs}
\begin{tabular}{llcc}
\toprule
\textbf{Dataset} & \textbf{Schema} & \textbf{Training Size} & \textbf{Total Records} \\
\midrule
California & Households $\rightarrow$ \{Individuals\} & 1000 & 3,808 \\
\midrule
\href{https://www.kaggle.com/datasets/agungpambudi/airline-loyalty-campaign-program-impact-on-flights/data}{Airlines} & Loyalty History $\rightarrow$ Activity & 1000 & 10,220 \\
\midrule
\href{https://www.kaggle.com/competitions/airbnb-recruiting-new-user-bookings}{Airbnb} & Users $\rightarrow$ Sessions & 1000 & 24,552 \\
\bottomrule
\end{tabular}
\end{table*}
\subsection{Section 3 Experiment Details}

We demonstrate how inter-table dependencies can leak membership information through a controlled experiment with synthetic relational data. This section provides complete details on the experimental setup and methodology referenced in the main paper.

We constructed a database with two tables: Customers (parent) and Transactions (child), connected through a one-to-many relationship. Both tables contained entities with 8-dimensional feature vectors sampled from a standard multivariate Gaussian distribution $N(0,I)$. The database contained no additional attributes beyond these feature vectors and the necessary primary/foreign keys establishing relationships between tables.

For our experiment, we generated 1000 customer entities and established different relationship patterns between members and non-members. Specifically, customer entities in the training set ($G_{mem}$) were each associated with 100 transactions, resulting in 100,000 total transaction records. In contrast, customer entities in the test set ($G_{non\text{-}mem}$) were each associated with exactly 1 transaction, resulting in 1000 total transaction records. The synthetic data generator ($G_{synth}$) was trained to mimic the training set, preserving the same structural patterns and feature distributions.

The evaluation dataset was balanced with a 50:50 ratio between member and non-member records. The membership inference task involved determining whether a customer record belonged to the training data used to generate $G_{synth}$.

\subsection{Section 5 Experiment}

\subsubsection{Datasets}
We evaluated MT-MIA on three benchmark multi-table relational datasets with different schema structures and entity relationships, as summarized in Table \ref{tab:dataset_configs}.

\subsection{Membership Inference Attack Descriptions}
\textbf{Distance to Closest Record (DCR.} Distance-based membership inference attacks \cite{ganleaks} are based on the intuition that synthetic data models may memorize training examples, leading to synthetic samples that lie closer in feature space to training members than to non-members. The Distance to Closest Record (DCR) method \cite{ganleaks} formalizes this intuition by defining
$f_{\text{DCR}}(x^*,S) = -\min_{{x} \in S} d(x^*,{x})$,
where $d(\cdot,\cdot)$ is a chosen distance metric. 

\textbf{Density Estimation.} In line with the memorization hypothesis of \cite{ganleaks}, \cite{houssiau2022tapas} and \cite{vanbreugel2023membership} present a simple strategy of rather than computing a distance, instead estimating the density of $x^*$ over the synthetic dataset: $f_{\text{Density Estimate}}(x^*,S) = {p_S(x^*)}$ using a Kernel Density Estimator.

\textbf{Monte Carlo (MC).} The Monte Carlo attack \cite{Hilprecht2019MonteCA} probes overfitting by counting how often synthetic samples fall near a query. Defining the $\varepsilon$-neighborhood around $x^*$ as $U_\varepsilon(x^*) = {x' \mid d(x^*, x') \leq \varepsilon}$, the method estimates the probability mass in this region by drawing $n$ samples $s_1,\ldots,s_n$ from $S$ and computing
$f_{\text{MC}}(x^*,S) = \frac{1}{n} \sum_{i=1}^{n} \mathbb{I}(s_i \in U_\varepsilon(x^*))$.
\subsubsection{Methodology}
For our experiments, we defined a "user subgraph" as the complete disjoint subgraph centered around a single user entity (parent node) in the relational database, including all its connected child entities across tables. This approach allows us to sample coherent relational data structures that maintain referential integrity.

To ensure proper evaluation of membership inference, we constructed training and holdout sets by sampling entirely disjoint user subgraphs. This sampling strategy guarantees that no entity (whether parent or child) appears in both the training and holdout sets, eliminating any potential data leakage during evaluation. We only included users with at least one node in each table of the dataset's schema to ensure consistent relational structures across all sampled subgraphs.

For all experiments, we employed ClavaDDPM and RealTabFormer with default hyperparameters as implemented in the original paper. 
No dataset-specific modifications or hyperparameter tuning was performed, as our goal was to evaluate MT-MIA's effectiveness under standard synthetic data generation conditions rather than optimizing synthetic data quality for each specific dataset.

We applied consistent feature processing across all datasets to prepare the data for both synthetic generation and attack model training:

\textbf{Feature Selection}: We excluded uninformative features such as ID columns from the feature space, as these are already represented in the graph structure. Date-time variables and open text fields were also dropped due to their high dimensionality and sparsity.

\textbf{Categorical Features}: All categorical variables were ordinal encoded before training.

\textbf{Numeric Features}: All numeric features were scaled using standard scaling with parameters fit to the synthetic data and applied consistently across real and synthetic samples.

\textbf{Missing Values}: Missing values were ordinally  encoded with 0s for categorical data or the feature mean for numeric values.

\section{Additional Tables and Figures}

\begin{table*}[h]
\caption{Notation summary for Section~\ref{sec:methodology}.}

\centering
\small
\begin{tabular}{ll}
\toprule
\textbf{Symbol} & \textbf{Meaning} \\
\midrule
\multicolumn{2}{l}{\emph{Encoder (Sec.~\ref{sec:encoder})}} \\
$\mathcal{M}_\theta$ & Heterogeneous graph encoder with parameters $\theta$ \\
$L$ & Number of message passing layers \\
$H_l^{(t)}$ & Layer $l$ embeddings for nodes of type $t$ \\
$\Phi_l^{(t)}$ & Type specific message passing function at layer $l$ \\
$A^{(r)}$ & Adjacency matrix for relation type $r$ \\
$z_{\text{parent}}$, $z_{\text{context}}$, $z_{\text{final}}$ & Parent, relational context, and fused embeddings \\
$g$ & Gating vector in $[0,1]^d$ \\
$\sigma$ & Elementwise sigmoid activation \\
$\text{MLP}_{\text{gate}}$ & Multilayer perceptron producing the gate \\
$\varphi$ & Non linear transformation in the gated residual \\
$\odot$ & Hadamard (elementwise) product \\
\midrule
\multicolumn{2}{l}{\emph{Training (Sec.~\ref{sec:training})}} \\
$\mathcal{D}_\phi, \mathcal{D}_\psi$ & Parent and context reconstruction decoders \\
$\phi, \psi$ & Decoder parameters \\
$\mathcal{N}(i)$ & Neighbor set of node $i$ across all adjacent node types \\
$\lambda_p, \lambda_c$ & Reconstruction loss weights \\
$\mathcal{L}_{\text{recon}}$ & Composite reconstruction loss \\
$f(h^*)$ & Membership scoring function \\
\bottomrule
\end{tabular}
\label{tab:notation}
\end{table*}
\label{app:add-results}
\begin{table*}[t]
  \centering
  \small
  \setlength{\tabcolsep}{4pt}
  \caption{Per-seed mean and standard deviation for the comparison of MT-MIA against the best baseline attack. Values are reported as $\text{mean}_{\pm\sigma}$ over three seeds. For each (model, dataset, metric) row, \emph{Baseline} reports the score of the strongest baseline attack on that configuration. This is the full-precision version of Table~\ref{tab:mt-mia-vs-baseline}.}
  \label{tab:mt-mia-vs-baseline-full}
  \begin{tabular}{llcccccc}
    \toprule
    Model & Metric & \multicolumn{2}{c}{California} & \multicolumn{2}{c}{Airbnb} & \multicolumn{2}{c}{Airlines} \\
    \cmidrule(lr){3-4} \cmidrule(lr){5-6} \cmidrule(lr){7-8}
     &  & Baseline & MT-MIA & Baseline & MT-MIA & Baseline & MT-MIA \\
    \midrule
    \textsc{ClavaDDPM} & AUC-ROC & $0.79_{\,\pm0.01}$ & $0.69_{\,\pm0.05}$ & $0.79_{\,\pm0.01}$ & $0.80_{\,\pm0.01}$ & $0.69_{\,\pm0.00}$ & $0.66_{\,\pm0.01}$ \\
     & TPR@FPR$=$0 & $0.00_{\,\pm0.00}$ & $0.07_{\,\pm0.09}$ & $0.00_{\,\pm0.00}$ & $0.01_{\,\pm0.00}$ & $0.07_{\,\pm0.01}$ & $0.17_{\,\pm0.03}$ \\
     & TPR@FPR$=$0.001 & $0.01_{\,\pm0.00}$ & $0.09_{\,\pm0.08}$ & $0.01_{\,\pm0.00}$ & $0.01_{\,\pm0.00}$ & $0.08_{\,\pm0.01}$ & $0.21_{\,\pm0.06}$ \\
     & TPR@FPR$=$0.01 & $0.11_{\,\pm0.02}$ & $0.15_{\,\pm0.08}$ & $0.06_{\,\pm0.01}$ & $0.09_{\,\pm0.01}$ & $0.12_{\,\pm0.00}$ & $0.31_{\,\pm0.05}$ \\
    \midrule
    \textsc{RelDiff} & AUC-ROC & $0.67_{\,\pm0.01}$ & $0.64_{\,\pm0.01}$ & $0.57_{\,\pm0.06}$ & $0.62_{\,\pm0.01}$ & $0.51_{\,\pm0.00}$ & $0.49_{\,\pm0.01}$ \\
     & TPR@FPR$=$0 & $0.00_{\,\pm0.00}$ & $0.15_{\,\pm0.03}$ & $0.00_{\,\pm0.00}$ & $0.03_{\,\pm0.02}$ & $0.00_{\,\pm0.00}$ & $0.00_{\,\pm0.00}$ \\
     & TPR@FPR$=$0.001 & $0.01_{\,\pm0.00}$ & $0.17_{\,\pm0.01}$ & $0.00_{\,\pm0.00}$ & $0.03_{\,\pm0.02}$ & $0.00_{\,\pm0.00}$ & $0.00_{\,\pm0.00}$ \\
     & TPR@FPR$=$0.01 & $0.13_{\,\pm0.01}$ & $0.22_{\,\pm0.01}$ & $0.02_{\,\pm0.01}$ & $0.06_{\,\pm0.01}$ & $0.01_{\,\pm0.00}$ & $0.01_{\,\pm0.00}$ \\
    \midrule
    \textsc{RTF} & AUC-ROC & $0.57_{\,\pm0.01}$ & $0.52_{\,\pm0.02}$ & $0.58_{\,\pm0.00}$ & $0.58_{\,\pm0.00}$ & $0.54_{\,\pm0.00}$ & $0.51_{\,\pm0.01}$ \\
     & TPR@FPR$=$0 & $0.00_{\,\pm0.00}$ & $0.00_{\,\pm0.01}$ & $0.00_{\,\pm0.00}$ & $0.00_{\,\pm0.00}$ & $0.01_{\,\pm0.01}$ & $0.01_{\,\pm0.00}$ \\
     & TPR@FPR$=$0.001 & $0.00_{\,\pm0.00}$ & $0.00_{\,\pm0.00}$ & $0.01_{\,\pm0.00}$ & $0.00_{\,\pm0.00}$ & $0.02_{\,\pm0.00}$ & $0.01_{\,\pm0.01}$ \\
     & TPR@FPR$=$0.01 & $0.02_{\,\pm0.00}$ & $0.02_{\,\pm0.01}$ & $0.02_{\,\pm0.00}$ & $0.02_{\,\pm0.01}$ & $0.05_{\,\pm0.02}$ & $0.03_{\,\pm0.01}$ \\
    \bottomrule
  \end{tabular}
\end{table*}

\begin{figure*}
    \centering
    \includegraphics[width=1\linewidth]{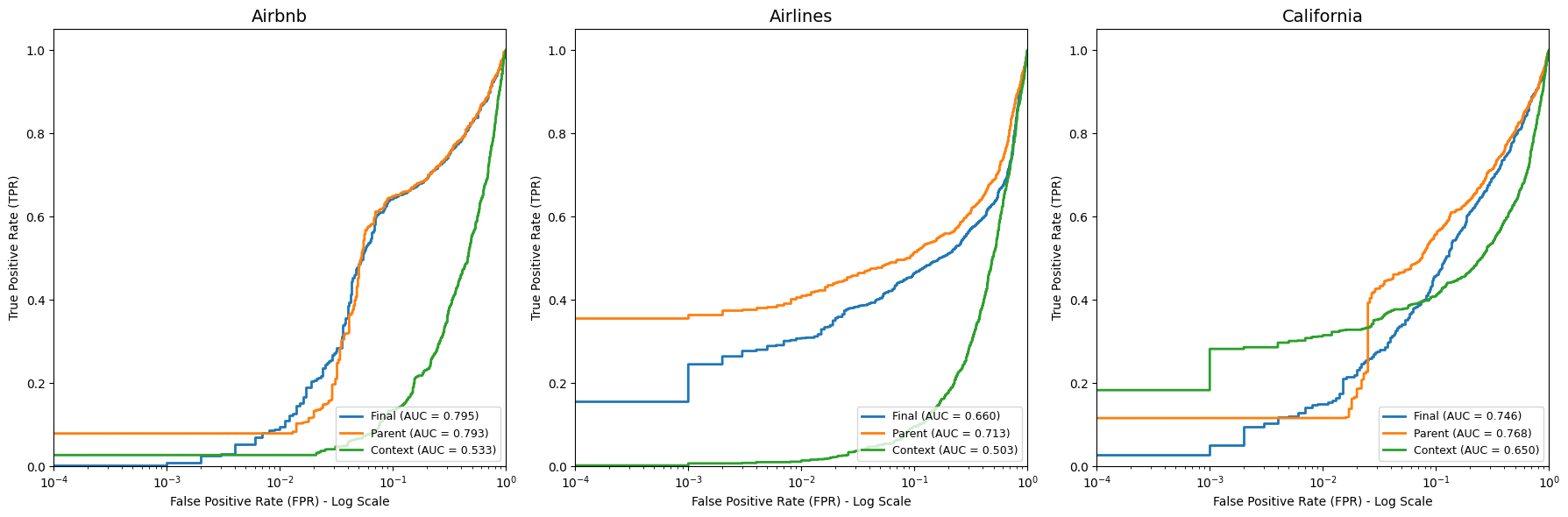}
    \caption{True Positive Rate by Log Scaled False Positive Rate for the most successful MT-MIA runs on ClavaDDPM. We plot the success of DCR utilizing the intermediate embeddings {$z_{parent}$} and  {$z_{context}$} as well as the final embedding {$z_{final}$}. While $z_{final}$ yields a competitive attack on high fidelity multi-table data, different datasets exhibit different sources of more severe privacy leakage that are capture in these intermediate latent spaces.}
    \label{fig:decomp}
    \Description{Three line plots side by side, one for each dataset: Airbnb, Airlines, and California. Each plot shows True Positive Rate on the y-axis (linear scale, 0 to 1) versus False Positive Rate on the x-axis (logarithmic scale, from 10 to the negative 4 to 1). Each plot contains three lines corresponding to attacks on different MT-MIA embeddings: Final (z-final), Parent (z-parent), and Context (z-context), with their AUC values reported in the legend. In the Airbnb plot, the Final and Parent lines (AUC 0.795 and 0.793) track closely together and rise from low True Positive Rate to near 1.0 across the False Positive Rate range, while the Context line (AUC 0.533) lies near the diagonal, indicating weak performance. In the Airlines plot, the Parent line (AUC 0.713) rises sharply at low False Positive Rate to a True Positive Rate near 0.4 and stays high across the range, the Final line (AUC 0.660) follows below it, and the Context line (AUC 0.503) lies near the diagonal. In the California plot, the Context line (AUC 0.650) rises sharply to a True Positive Rate near 0.2 at the lowest False Positive Rates and continues climbing, outperforming both the Final line (AUC 0.746) and the Parent line (AUC 0.768) in the low False Positive Rate region, despite having lower overall AUC.}
\end{figure*}
\end{document}